\documentclass{article}

\usepackage{microtype}
\usepackage{graphicx}
\usepackage{subfigure}
\usepackage{booktabs} 

\usepackage{natbib}
\usepackage{multibib}
\newcites{New}{References}

\usepackage{amsmath, amssymb, amsthm, amsfonts}
\usepackage{enumitem}
\usepackage{bm, bbm}
\usepackage{hhline}
\usepackage{mathtools, mathrsfs}
\usepackage{graphicx}
\usepackage{subfigure}
\usepackage{xcolor}
\usepackage{booktabs}
\let\oldsquare\square    
\usepackage{mathabx}
\usepackage{algorithm}
\usepackage[titlenumbered,ruled,noend,algo2e]{algorithm2e}

\SetCommentSty{mycommfont}
\SetEndCharOfAlgoLine{}

\usepackage{hyperref}
\usepackage{cleveref}
\usepackage{commands}

\usepackage[accepted]{icml2020}

\hypersetup{
    colorlinks,
    linkcolor={red!75!black},
    citecolor={blue!80!black},
    urlcolor={blue!80!black}
}

\icmltitlerunning{Duality in RKHSs with Infinite Dimensional Outputs: Application to Robust Losses}

\begin{document}

\twocolumn[
\icmltitle{Duality in RKHSs with Infinite Dimensional Outputs:\\Application to Robust Losses}



\icmlsetsymbol{equal}{*}

\begin{icmlauthorlist}
\icmlauthor{Pierre Laforgue}{to}
\icmlauthor{Alex Lambert}{to}
\icmlauthor{Luc Brogat-Motte}{to}
\icmlauthor{Florence d'Alch\'{e}-Buc}{to}
\end{icmlauthorlist}

\icmlaffiliation{to}{LTCI, T\'{e}l\'{e}com Paris, Institut Polytechnique de Paris, France}

\icmlcorrespondingauthor{Pierre Laforgue}{pierre.laforgue@telecom-paris.fr}
\icmlkeywords{Kernel Methods, Operator-Valued Kernels, Robustness, Structured Prediction}

\vskip 0.3in
]



\printAffiliationsAndNotice{}  

\begin{abstract}
Operator-Valued Kernels (OVKs) and associated vector-valued Reproducing Kernel Hilbert Spaces provide an elegant way to extend scalar kernel methods when the output space is a Hilbert space.
Although primarily used in finite dimension for problems like multi-task regression, the ability of this framework to deal with infinite dimensional output spaces unlocks many more applications, such as functional regression, structured output prediction, and structured data \mbox{representation}.
However, these sophisticated schemes crucially rely on the kernel trick in the output space, so that most of previous works have focused on the square norm loss function, completely neglecting robustness issues that may arise in such surrogate problems.
To overcome this limitation, this paper develops a duality approach that allows to solve OVK machines for a wide range of loss functions.
The infinite dimensional Lagrange multipliers are handled through a \emph{Double Representer Theorem}, and algorithms for $\epsilon$-insensitive losses and the Huber loss are thoroughly detailed.
Robustness benefits are emphasized by a theoretical stability analysis, as well as empirical improvements on structured data applications.
\end{abstract}

\section{Introduction}
\label{sec:intro}

Due to increasingly available streaming and network data, learning to predict complex objects such as \mbox{structured} \mbox{outputs} or time series has attracted a great deal of \mbox{attention} in machine learning. Extending the well known kernel methods devoted to non-vectorial data \cite{Hofmann08}, several kernel-based approaches have emerged~to deal with complex output data.
While Structural SVM and variants cope with discrete structures \citep{tsochantaridis2005large,Joachims09} through structured losses, \mbox{Operator-Valued} Kernels (OVKs) and vector-valued \mbox{Reproducing} Kernel Hilbert Spaces (vv-RKHSs, \citet{micchelli2005learning,carmeli2006vector,carmeli2010vector}) provide a unique framework to handle both functional and \mbox{structured} outputs.
Vv-RKHSs are classes of functions that map an \mbox{arbitrary} input set $\mathcal{X}$ to some output Hilbert space $\bmY$ \citep{Senkene1973, Caponnetto2008}.
Primarily used with finite dimensional outputs ($\bmY=\mathbb{R}^p$) to solve multi-task regression \citep{micchelli2005learning,Baldassare2012} and multiple class classification \citep{Dinuzzo2011}, OVK methods have further been exploited to handle outputs in infinite dimensional Hilbert spaces.
This has unlocked numerous applications, such as functional regression \citep{kadri2010,kadri2016ovk}, structured prediction \citep{Brouard_icml11,Kadri_icml2013}, infinite quantile regression \citep{brault2019infinite}, or structured data representation learning  \citep{laforgue2019autoencoding}.
Nonetheless, these sophisticated schemes often come along with a basic loss function: the output space squared norm, neglecting desirable properties such as parsimony and robustness.

In nonparametric modeling, model parsimony boils down to data sparsity, \textit{e.g.} reducing the number of training data points on which the model relies to make a prediction. Such a property is highly valuable \cite{hastie2015statistical}: not only does it prevent overfitting but it also alleviates the inherent computational load of optimization and prediction, allowing to scale to larger datasets.
Another appealing property of a regression tool is robustness to outliers \cite{huber1964robust,Zhu2008}.
Real data may suffer from incorrect feature measurements and spurious annotations, leading to training datasets contaminated with outliers.
Then, minimizing the squared loss is inappropriate as the least-squares estimates behave poorly when the residuals distribution is not normal, but rather heavy-tailed.
In (scalar) kernel methods, these two properties -- data sparsity and robustness to outliers -- are imposed through the choice of appropriate losses.
Data sparsity is leveraged by using $\epsilon$-insensitive losses, exploited in the well known Support Vector Regression \cite{drucker1997support} while robust regression \cite{Fung2000} can be obtained by minimizing the Huber loss function \citep{huber1964robust}.
Driven by three emblematic learning tasks, structured prediction, functional regression, and structured data representation, we propose a general duality framework that enables sparse data regression and robust regression, even when working in vv-RKHSs with infinite-dimensional outputs.
Although extensively used within scalar kernel methods, very few attempts have been made to adapt duality to vv-RKHSs.
In \citet{brouard2016input}, dualization is presented, but only used in the maximum margin regression scenario.
\citet{sangnier2017data} consider a wider class of loss functions, including $\epsilon$-insensitive losses to leverage data sparsity, but only in the case of matrix-valued kernels \citep{Alvarez2012}, for which the dual problem is finite dimensional.
For a general OVK however, the dual problem is to be solved over $\mathcal{Y}^n$, and is intractable without additional work when $\mathcal{Y}$ is infinite dimensional.
We first notice that the extensions of $\epsilon$-insensitive losses and the Huber loss to general Hilbert space are (still) expressed as convolutions of simpler losses whose Fenchel-Legendre (FL) transforms are known.
\mbox{Inspired} by this remark, we identify general conditions on the OVKs and FL transforms to establish a \emph{Double Representer Theorem} allowing to work with matrix parameterized representations.
In particular, a careful use of the duality principle considerably broadens the range of loss functions for which OVK solutions are computable.
The present work thus aims at developing a comprehensive methodology to solve these dual problems.

The rest of the paper is organized as follows.
In \Cref{sec:preliminaries}, we introduce OVKs, recall the general formulation of dual problems for OVK machines, and derive their solvable finite dimensional reformulation.
\Cref{sec:3-robust} is devoted to \mbox{specific} instantiations of this problem for $\epsilon$-insensitive losses and the Huber loss, with algorithms duly explicited.
In \Cref{sec:5-applis}, we apply our framework to induce sparsity and robustness into structured prediction, functional regression, and structured data representation.
Proofs are postponed to the Appendix.

\section{Learning in vv-RKHSs}
\label{sec:preliminaries}

After reminders on OVKs and vv-RKHS learning theory, this section exposes the duality approach for the regularized empirical risk minimization problem in vv-RKHSs.
Two strategies are then detailed to solve infinite dimensional dual problems, either under an assumption on the kernel, or by approximating the dual.
In the following, $\mathcal{Y}$ is assumed to be a separable Hilbert space.
\begin{definition}\label{def:ovk}
An OVK is an application $\mathcal{K} \colon \mathcal{X} \times \mathcal{X} \to \mathcal{L}(\mathcal{Y})$, that satisfies the following two properties for all $n \in \mathbb{N}^*$:\vspace{-0.05cm}
\begin{align*}
1)& ~\forall (x, x') \in \mathcal{X} \times \mathcal{X}, \qquad\quad~~ \mathcal{K}(x, x') = \mathcal{K}(x', x)^\#,\\
2)& ~\forall~(x_i, y_i)_{i =1}^n \in (\mathcal{X} \times \mathcal{Y})^n, \sum_{i, j = 1}^n \langle y_i, \mathcal{K}(x_i, x_j)y_j \rangle_\mathcal{Y} \ge 0,
\end{align*}
with $\mathcal{L}(E)$ the set of bounded linear operators on vector space $E$, and $A^\#$ the adjoint of any operator $A$.
\end{definition}

A simple example of OVK is the \emph{separable kernel}.
\begin{definition}\label{def:decompo_ovk}
$\mathcal{K}: \mathcal{X} \times \mathcal{X} \rightarrow \mathcal{L}(\mathcal{Y})$ is a \emph{separable kernel} iff there exist a scalar kernel $k: \mathcal{X} \times \mathcal{X} \rightarrow \mathbb{R}$ and a positive semi-definite operator $A \in \mathcal{L}(\mathcal{Y})$ such that for all $(x, x') \in \mathcal{X}^2$ it holds:\quad $\mathcal{K}(x,x')=k(x,x')A$.
\end{definition}

Similarly to scalar-valued kernels, an OVK can be uniquely associated to a functional space from $\mathcal{X}$ to $\mathcal{Y}$: its vv-RKHS.
\begin{theorem}
Let $\mathcal{K}$ be an OVK, and for $x \in \mathcal{X}$, let $\mathcal{K}_x \colon y \mapsto \mathcal{K}_xy \in \mathcal{F}(\mathcal{X}, \mathcal{Y})$ the linear operator such that: $\forall x' \in \mathcal{X},~(\mathcal{K}_x y)(x') = \mathcal{K}(x', x)y$.
Then, there is a unique Hilbert space $\mathcal{H}_\mathcal{K} \subset \mathcal{F}(\mathcal{X}, \mathcal{Y})$ the vv-RKHS associated to $\mathcal{K}$ such that $\forall x \in \mathcal{X}$ it holds:\vspace{-0.3cm}
\begin{enumerate}[label=(\roman*)]
\item $\mathcal{K}_x$ spans the space $\mathcal{H}_\mathcal{K}$ ($\forall y \in \mathcal{Y} \colon \mathcal{K}_x y \in \mathcal{H}_{\mathcal{K}}$)\\[-0.5cm]
\item $\mathcal{K}_x$ is bounded for the uniform norm\\[-0.5cm]
\item $\forall f \in \mathcal{H}_\mathcal{K},~f(x) = \mathcal{K}_{x}^\# f$ (reproducing property)
\end{enumerate}
\end{theorem}

Given a sample $\sample = \{(x_i, y_i)_{i = 1}^n\} \in (\mathcal{X}\times\mathcal{Y})^n$ of $n$ i.i.d. realizations of a generic random variable $(X, Y)$, an OVK $\mathcal{K}: \mathcal{X} \times \mathcal{X} \rightarrow \mathcal{L}(\mathcal{Y})$, a convex loss function $\ell: \mathcal{Y} \times \mathcal{Y} \rightarrow \mathbb{R}$, and a regularization parameter $\Lambda > 0$,  the general form of an OVK-based learning problem is to find $\hat{h}$ that solves:
\begin{problem}\label{pbm:primal_pbm}
\min_{h \in \mathcal{H}_\mathcal{K}} ~ \frac{1}{n} \sum_{i=1}^n \ell(h(x_i), y_i) + \frac{\Lambda}{2} \|h\|_{\mathcal{H}_\mathcal{K}}^2.
\end{problem}
Similarly to scalar ones, a crucial tool in operator-valued kernel methods is the \mbox{\emph{Representer Theorem}}, ensuring that $\hat{h}$ actually pertains to a reduced subspace of $\mathcal{H}_\mathcal{K}$.\vspace{0.2cm}
\begin{theorem}\label{thm:rt}
(Theorem 4.2 in \citet{micchelli2005learning})
There exists $(\hat{\alpha}_i)_{i = 1}^n \in \mathcal{Y}^n$ such that
\begin{equation*}
  \hat{h} = \frac{1}{\Lambda n}\sum_{i=1}^n \mathcal{K}(\cdot, x_i)\hat{\alpha}_i.
\end{equation*}
\end{theorem}
Although \Cref{thm:rt} drastically downscales the search \mbox{domain} (from $\mathcal{H}_\mathcal{K}$ to $\mathcal{Y}^n$), it gives no further information about the $(\hat{\alpha}_i)_{i = 1}^n$.
One way to gain insight about these \mbox{coefficients} is to perform \Cref{pbm:primal_pbm}'s dualization, with the notation $\ell_i: y\in\mathcal{Y} \mapsto \ell(y, y_i)$ for any $i \le n$.

\begin{theorem}\label{thm:dual_pbm}
(Appendix B in \citet{brouard2016input})
The solution to \Cref{pbm:primal_pbm} is given by\vspace{-0.1cm}
\begin{equation}\label{eq:decompo_estim_dual}
\hat{h} = \frac{1}{\Lambda n}\sum_{i=1}^n \mathcal{K}(\cdot, x_i)\hat{\alpha}_i,
\end{equation}\vspace{-0.1cm}
with $(\hat{\alpha}_i)_{i = 1}^n \in \mathcal{Y}^n$ the solutions to the dual problem\vspace{-0.2cm}
\begin{problem}\label{pbm:dual_pbm}
\min_{(\alpha_i)_{i=1}^n\in \mathcal{Y}^n} ~ \sum_{i=1}^n\ell_i^\star(-\alpha_i) + \frac{1}{2\Lambda n} \sum_{i,j=1}^n \left\langle \alpha_i, \mathcal{K}(x_i, x_j) \alpha_j\right\rangle_\mathcal{Y},
\end{problem}
where $f^\star: \alpha \in \mathcal{Y} \mapsto \sup_{y \in \mathcal{Y}} \left\langle \alpha, y\right\rangle_\mathcal{Y} - f(y)$ denotes the Fenchel-Legendre transform of a function $f:\mathcal{Y} \rightarrow \mathbb{R}$.
\end{theorem}
Refer to \Cref{sec:apx_dual} for \Cref{thm:dual_pbm}'s proof, that has been reproduced for self-containedness.
Dualization brings in additional information about the optimal coefficients (notice nonetheless that \Cref{thm:rt} holds true for a much wider class of problems).
As it is, \Cref{pbm:dual_pbm} is however of little interest, since the optimization must be performed on the \mbox{infinite} dimensional space $\mathcal{Y}^n$.
\mbox{Depending} on the problem, we propose two solutions: either using a \emph{Double Representer Theorem}, or by approximating \Cref{pbm:dual_pbm}.
\par{\bf Notation.}{
If $\mathcal{K}$ is identity decomposable (i.e. $\mathcal{K}=k~\mathbf{I}_\mathcal{Y}$), $K^X$ and $K^Y$ denote the input and output gram matrices.
For any matrix $M$, $M_{i:}$ represents its $i^{th}$ line, and $\|M\|_{p, q}$ its $\ell_{p, q}$ row wise mixed norm, \textit{i.e.} the $\ell_q$ norm of the $\ell_p$ norms of its lines.
$\chi_S$ denotes the characteristic function of a set $S$, null on $S$ and equal to $+\infty$ otherwise, $f \infconv g$ is the infimal convolution of $f$ and $g$ \citep{bauschke2011convex}, $(f \infconv g)(x) = \inf_y f(y) + g(x - y)$.
Finally, $\#S$ is the cardinality of any set $S$, and $\|\cdot\|_{\text{op}}$ the operator norm.
}


\subsection{The Double Representer Theorem}
\label{sec:kae_rudi}

In order to make \Cref{pbm:dual_pbm} solvable, we need assumptions on the loss and the kernel.
Let $\bm{\mathsf{Y}}$ denote $\text{span}(y_i,~i \le n)$.
\Cref{hyp:FL_1,hyp:FL_2} characterize admissible losses through conditions on their Fenchel-Legendre (FL) transforms.
They are standard for kernel methods, and ensure computability by stipulating that only dot products are involved.

\begin{assumption}\label{hyp:FL_1}
$\forall i \le n,~\forall (\alpha^{\bm{\mathsf{Y}}}, \alpha^\perp) \in \bm{\mathsf{Y}} \times \bm{\mathsf{Y}}^\perp$, it holds
$\ell_i^\star(\alpha^{\bm{\mathsf{Y}}}) \le \ell_i^\star(\alpha^{\bm{\mathsf{Y}}} + \alpha^\perp)$.
\end{assumption}

\begin{assumption}\label{hyp:FL_2}
$\forall i \le n, \exists L_i: \mathbb{R}^{n + n^2} \rightarrow \mathbb{R}$ such that for all $\bm{\omega} = (\omega_j)_{j \le n} \in \mathbb{R}^n$, \quad $\textstyle \ell_i^\star(-\sum_{j=1}^n \omega_j~y_j ) = L_i(\bm{\omega}, K^Y)$.
\end{assumption}

Regarding the OVK, the key point is \Cref{hyp:stable_ovk}.
Roughly speaking, $\bm{\mathsf{Y}}$ is what we \emph{see} and \emph{know} about output space $\mathcal{Y}$, while $\bm{\mathsf{Y}}^\perp$ represents the part we \emph{ignore}.
What we need is an OVK somewhat \emph{aligned} with the outputs, in the sense that the little we know about $\mathcal{Y}$ should be preserved through $\mathcal{K}$.
As for \Cref{hyp:sum_decompo}, it helps simplifying the computations.

\begin{assumption}\label{hyp:stable_ovk}
$\forall i, j \le n$, $\bm{\mathsf{Y}}$ is invariant by $\mathcal{K}(x_i, x_j)$, \textit{i.e.} $\forall y \in \mathcal{Y},~~y \in \bm{\mathsf{Y}} \Rightarrow \mathcal{K}(x_i, x_j)y \in \bm{\mathsf{Y}}$.
\end{assumption}

\begin{remark}\label{rmk:stable_ovk}
It is important to notice that we do not need \Cref{hyp:stable_ovk} to hold true for every collection $\{y_i\}_{i \le n} \in \mathcal{Y}^n$.
It rather constitutes an a posteriori condition to ensure that the kernel is aligned with the training sample at hand.
If $\mathcal{Y}$ is finite dimensional, one may hope that with sufficiently many outputs, then $\bm{\mathsf{Y}}$ spans $\mathcal{Y}$, and every matrix-valued kernel then fits.
If $\mathcal{Y}$ is infinite dimensional, identity-decomposable kernels are admissible (which despite simple expressions may describe nontrivial dependences in infinite dimensional spaces).
Moreover, separable kernels with operators similar to the empirical covariance $\sum_i y_i\otimes y_i$ \citep{Kadri_icml2013} are also eligible, opening the door to ad-hoc and learned kernels, see \Cref{apx:suit_ker} for further examples.
\end{remark}

\begin{assumption}\label{hyp:sum_decompo}
There exist $T \ge 1$, and for every $t \le T$ admissible scalar kernels $k_t : \mathcal{X} \times \mathcal{X} \rightarrow \mathbb{R}$ as well as positive semi-definite operators $A_t \in \mathcal{L}(\mathcal{Y})$, such that for all $(x, x') \in \mathcal{X}^2$ it holds:~~$
\mathcal{K}(x, x') = \sum_{t=1}^T k_t(x, x')A_t$.
\end{assumption}
Under \Cref{hyp:sum_decompo}, $K_t^X$ and $K_t^Y$ denote the matrices such that $[K_t^X]_{ij} = k_t(x_i, x_j)$, $[K_t^Y]_{ij} = \left\langle y_i, A_t y_j\right\rangle_\mathcal{Y}$.
Notice that it is by no means restrictive, since every shift-invariant OVK can be approximated arbitrarily closely by kernels satisfying \Cref{hyp:sum_decompo}.
Furthermore, if for all $t \le T$, $A_t$ keeps $\bm{\mathsf{Y}}$ invariant, then \Cref{hyp:stable_ovk} is directly fulfilled.
Under these assumptions, \Cref{thm:double_rt} proves that the optimal coefficients lie in $\bm{\mathsf{Y}}^n$, ensuring the solutions computability.

\begin{theorem}\label{thm:double_rt}
Let $\ell: \mathcal{Y} \times \mathcal{Y} \rightarrow \mathbb{R}$ be a loss function with Fenchel-Legendre transforms satisfying \Cref{hyp:FL_1,hyp:FL_2}, and $\mathcal{K}$ be an OVK verifying \Cref{hyp:stable_ovk}.
Then, the solution to \Cref{pbm:primal_pbm} is given by\vspace{-0.2cm}
\begin{equation}\label{eq:expansion}
\hat{h} = \frac{1}{\Lambda n}\sum_{i,j=1}^n \mathcal{K}(\cdot, x_i)~\hat{\omega}_{ij}~y_j,
\end{equation}
with $\hat{\Omega} = [\hat{\omega}_{ij}] \in \mathbb{R}^{n \times n}$ the solution to the dual problem
\begin{equation*}
\min_{\Omega \in \mathbb{R}^{n \times n}} ~ \sum_{i=1}^n L_i\left( \Omega_{i:}, K^Y \right) + \frac{1}{2 \Lambda n} \mathbf{Tr}\left(\tilde{M}^\top (\Omega \otimes \Omega)\right),
\end{equation*}
with $M$ the $n^4$ tensor such that $M_{ijkl} = \langle y_k , \mathcal{K}(x_i, x_j)y_l\rangle_\mathcal{Y}$, and $\tilde{M}$ its rewriting as a $n^2 \times n^2$ block matrix.
If kernel $\mathcal{K}$ further satisfies \Cref{hyp:sum_decompo}, then tensor $M$ simplifies to $M_{ijkl} = \sum_{t=1}^T [K^X_t]_{ij} [K^Y_t]_{kl}$, and the problem rewrites
\begin{problem}\label{pbm:omega_pbm}
\min_{\Omega \in \mathbb{R}^{n \times n}} ~ \sum_{i=1}^n L_i\left( \Omega_{i:}, K^Y \right) + \frac{1}{2\Lambda n} \sum_{t=1}^T \mathbf{Tr}\left( K_t^X \Omega K_t^Y \Omega^\top\right).
\end{problem}
\end{theorem}

See \Cref{apx:double_rt} for the proof.
This theorem can be seen as a \emph{Double Representer Theorem}, since both theorems share analogous proofs and  consequences: a search domain reduction, respectively from $\mathcal{H}_\mathcal{K}$ to $\mathcal{Y}^n$, and $\mathcal{Y}^n$ to $\mathbb{R}^{n \times n}$.

\begin{remark}
The \emph{Double Representer Theorem} emphasizes that only the knowledge of the $n^4$ tensor $M$ is required to make OVK problems in infinite dimensional output spaces computable.
Although it might seem prohibitive at first sight, one has to keep in mind that, like for scalar kernel methods, a first $n^2$ cost is needed to use (input) kernels with infinite dimensional feature maps, while the second $n^2$ cost allows for handling infinite dimensional outputs.
In the case of a decomposable kernel, one has $M_{ijkl} = K^X_{ij} K^Y_{kl}$.
One only needs two $n^2$ matrices, recovering the scalar complexity.
\end{remark}

We now present a non-exhaustive list of admissible losses (one may refer to \Cref{apx:good_losses} for the proof).
\begin{proposition}\label{prop:good_losses}
The following losses have Fenchel-Legendre transforms verifying \mbox{Assumptions} \ref{hyp:FL_1} and \ref{hyp:FL_2}:
\begin{itemize}
\item $\ell_i(y) = f(\left\langle y, z_i\right\rangle)$, $z_i \in Y$ and $f:\mathbb{R} \rightarrow \mathbb{R}$ convex. This encompasses maximum-margin regression, obtained with $z_i = y_i$ and $f(t) = \max(0, 1 -t)$.
\item $\ell(y) = f(\|y\|)$, $f:\mathbb{R}_+ \rightarrow \mathbb{R}$ convex increasing s.t. $t \mapsto \frac{f'(t)}{t}$ is continuous over $\mathbb{R}_+$. This includes all power functions $\frac{\lambda}{\eta}\|y\|_\mathcal{Y}^\eta$ for $\eta > 1$ and $\lambda > 0$.
\item $\forall \lambda > 0$, with $\mathcal{B}_\lambda$ the centered ball of radius $\lambda$,
      \begin{align*}
      &\sqbullet~\ell(y) = \lambda \|y\|,                 && \sqbullet~\ell(y) = \lambda \|y\|\log(\|y\|),\\
      &\sqbullet~\ell(y) = \chi_{\mathcal{B}_\lambda}(y), && \sqbullet~\ell(y) = \lambda (\exp(\|y\|) - 1).
      \end{align*}
\item $\ell_i(y) = f(y-y_i)$, $f^\star$ verifying \Cref{hyp:FL_1,hyp:FL_2}.
\item Any infimal convolution involving functions satisfying \Cref{hyp:FL_1,hyp:FL_2}.
This encompasses $\epsilon$-insensitive losses \citep{sangnier2017data}, the Huber loss \citep{huber1964robust}, and generally all Moreau or Pasch-Hausdorff envelopes \citep{moreau1962fonctions,bauschke2011convex}.
\end{itemize}
\end{proposition}


\subsection{Approximating the Dual Problem}
\label{sec:dual_approx}

If \Cref{hyp:stable_ovk} is not satisfied, another way to get a finite dimensional decomposition similar to that of \Cref{thm:double_rt} is to approximate the dual problem.
This may be done by restricting the dual variables to suitable finite dimensional subsets of $\mathcal{Y}$, if the following hypothesis on kernel $\mathcal{K}$ holds.

\begin{assumption}\label{hyp:compact_self}
The kernel $\mathcal{K}=k \cdot A$ is a separable OVK, with $A$ a compact operator.
\end{assumption}

Recalling that $A$ is by design self adjoint and positive, its compactness then allows for a spectral decomposition: there exists an orthonormal basis $(\psi_j)_{j=1}^\infty$ of $\mathcal{Y}$, and some positive $(\lambda_j)_{j=1}^\infty$, ordered in a non-increasing fashion and converging to zero, such that $A = \sum_{j=1}^\infty\lambda_j \psi_j \otimes \psi_j$ \citep{osborn1975spectral}.


Using such a basis, one can say that there exists $(\hat{\omega}_i)_{i=1}^n \in \ell^2(\mathbb{R})^n$ such that $\forall i \leq n, \hat{\alpha}_i = \sum_{j=1}^\infty \hat{\omega}_{ij} \psi_j$.
Since this leads to an infinite size representation of the dual variables, the idea is then to restrict the search space to the eigenspace associated to the $m$ largest eigenvalues of $A$, for some $m>0$.
Let $\widetilde{\mathcal{Y}}_m$ denote $\text{span}(\{\psi_j\}_{j=1}^m)$, and $S=\text{diag}(\lambda_j)_{j=1}^m$.
An approximated dual problem reads
\begin{problem}\label{pbm:dual_pbm_approx}
\min_{(\alpha_i)_{i=1}^n\in \widetilde{\mathcal{Y}}_m^n} ~ \sum_{i=1}^n\ell_i^\star(-\alpha_i) + \frac{1}{2\Lambda n} \sum_{i,j=1}^n \left\langle \alpha_i, \mathcal{K}(x_i, x_j) \alpha_j\right\rangle_\mathcal{Y},
\end{problem}

We now state a condition similar to \Cref{hyp:FL_2}, which makes the solution to \Cref{pbm:dual_pbm_approx} computable.
\begin{assumption}\label{hyp:FL_3}
$\forall i \le n, \exists L_i: \mathbb{R}^{2m} \rightarrow \mathbb{R}$ such that \mbox{$\forall~\bm{\omega} = (\omega_j)_{j \le m} \in \mathbb{R}^m$}, ~~ $\textstyle \ell_i^\star(-\sum_{j=1}^m \omega_j~\psi_j ) = L_i(\bm{\omega}, R_{i:})$, with $R \in \mathbb{R}^{n \times m}$ the matrix such that $R_{ij} = \langle y_i,\psi_j \rangle_{\mathcal{Y}}$.
\end{assumption}
\begin{remark}

  \Cref{hyp:FL_3} is similar to \Cref{hyp:FL_2}, except that the output Gram matrix $K^Y$ is replaced by matrix $R$ storing the dot products between the orthonormal family $\{\psi_j\}_{j=1}^m$ and the outputs. In particular, all losses explicited in \Cref{prop:good_losses} have FL transforms verifying \Cref{hyp:FL_3}.
\end{remark}
\begin{theorem}\label{thm:double_rt_approx}
Let $\mathcal{K}$ be an OVK meeting Assumption \ref{hyp:compact_self} and \mbox{$\ell: \mathcal{Y} \times \mathcal{Y} \rightarrow \mathbb{R}$} be a loss function with FL transforms \mbox{satisfying} Assumption \ref{hyp:FL_3}.
Then, \Cref{pbm:dual_pbm_approx} is equivalent to
\begin{problem}\label{pbm:omega_pbm_approx}
\min_{\Omega \in \mathbb{R}^{n \times m}} ~ \sum_{i=1}^n L_i\left( \Omega_{i:}, R_{i:} \right) + \frac{1}{2\Lambda n} \mathbf{Tr}\left( K^X \Omega S \Omega^\top\right).
\end{problem}
Denoting by $\hat{\Omega} = [\hat{\omega}_{ij}] \in \mathbb{R}^{n \times m}$ the solution to \Cref{pbm:omega_pbm_approx}, the associated predictor is finally given by
\begin{equation}\label{eq:expansion_approx}
\hat{h} = \frac{1}{\Lambda n}\sum_{i=1}^n \sum_{j=1}^{m} k(\cdot, x_i)~\lambda_j~\hat{\omega}_{ij}~\psi_j,
\end{equation}
\end{theorem}

\begin{remark}
The rationale behind the above approximation is that under compactness of $A$, \Cref{eq:expansion_approx} constitutes a reasonable approximation of \Cref{eq:decompo_estim_dual}.
Notice that \citet{kadri2016ovk} use a truncated spectral decomposition of the operator to implement a functional version of Kernel Ridge Regression, without resorting to dualization however.
\end{remark}

\section{Application to Robust Losses}
 \label{sec:3-robust}

We now instantiate \Cref{thm:double_rt}'s dual problem for three loss functions encouraging data sparsity and robustness. They write as infimal convolutions, and are thus hardly tractable in the primal.
Their dual problems enjoy simple resolution algorithms that are thoroughly detailed.
A stability analysis is also carried out to highlight the hyperparameters impact.

\subsection{Complete Dual Resolution for Three Robust Losses}

As a first go, we recall the important notion of $\epsilon$-insensitive losses.
Following in the footsteps of \citet{sangnier2017data}, we extend them in a natural way from $\mathbb{R}^p$ to any Hilbert space $\mathcal{Y}$.
To avoid additional notation, in this subsection $\ell$ denotes the loss taken w.r.t. one argument (previously $\ell_i$).

\begin{definition}\label{def:eps_insensitive}
Let $\ell: \mathcal{Y} \rightarrow \mathbb{R}_+$ be a convex loss such that $\ell(0) = 0$, and $\epsilon > 0$.
The $\epsilon$-insensitive version of $\ell$, denoted $\ell_\epsilon$, is defined by $\ell_\epsilon(y) = (\ell \infconv \chi_{\mathcal{B}_\epsilon})(y)$, or again:
\begin{equation*}
\forall y \in \mathcal{Y}, ~~ \ell_\epsilon(y) =
\left\{\begin{matrix*}
\begin{array}{cl}
0 & \textrm{if } \|y\|_\mathcal{Y} \le \epsilon\\[0.1cm]
\displaystyle \inf_{\|d\|_\mathcal{Y} \le 1} \ell(y - \epsilon d) & \textrm{otherwise}
\end{array}
\end{matrix*}\right..
\end{equation*}
\end{definition}

In other terms, $\ell_\epsilon(y)$ is the smallest value of $\ell$ within the ball of radius $\epsilon$ centered at $y$.
As revealed by the next definition, natural choices for $\ell$ yield extensions of celebrated scalar loss functions to infinite dimensional Hilbert spaces.

\begin{definition}\label{def:eps_svr_ridge}
If $\ell = \|\cdot\|_\mathcal{Y}$, then $\|\cdot\|_{\mathcal{Y}, \epsilon} = \max(\|\cdot\|_\mathcal{Y} - \epsilon, 0)$, and the related problem is the natural extension of $\epsilon$-SVR.\\[0.15cm]
If $\ell = \|\cdot\|_\mathcal{Y}^2$, then $\|\cdot\|_{\mathcal{Y}, \epsilon}^2 = \max(\|\cdot\|_\mathcal{Y} - \epsilon, 0)^2$, and the related problem is called the $\epsilon$-insensitive Ridge regression.
\end{definition}

The third framework that nicely falls into our resolution methodology is the Huber loss regression \citep{huber1964robust}.
Tailored to induce robustness, the Huber loss function does not feature convolution with $\chi_{\mathcal{B}_\epsilon}$ but rather between the first two powers of the Hilbert norm (that used in \Cref{def:eps_svr_ridge}).

\begin{definition}\label{def:Huber}
The Huber loss of parameter $\kappa$ is given by $\ell_{H, \kappa}(y) = (\kappa \|\cdot\|_\mathcal{Y} \infconv \frac{1}{2} \|\cdot\|_\mathcal{Y}^2)(y)$, or again:
\begin{equation*}
\forall y \in \mathcal{Y}, ~~ \ell_{H, \kappa}(y) =
\left\{\begin{matrix*}
\begin{array}{cl}
\frac{1}{2} \|y\|_\mathcal{Y}^2 & \textrm{if } \|y\|_\mathcal{Y} \le \kappa\\[0.15cm]
\kappa\left(\|y\|_\mathcal{Y} - \frac{\kappa}{2}\right) & \textrm{otherwise}
\end{array}
\end{matrix*}\right..
\end{equation*}
\end{definition}

Due to its asymptotic behavior as $\|\cdot\|_\mathcal{Y}$, the Huber loss is useful when the training data is heavy tailed or contains outliers.
Illustrations of \Cref{def:eps_svr_ridge,def:Huber}'s loss functions in one and two dimensions are available in \Cref{apx:loss_fig}.
Interestingly, \Cref{pbm:omega_pbm} for these three losses -- and an identity decomposable kernel -- admits a very nice writing, allowing for an efficient resolution.

\begin{theorem}\label{thm:loss_instantiation}
If $\mathcal{K} = k~\mathbf{I}_\mathcal{Y}$, the solutions to the $\epsilon$-Ridge regression, $\kappa$-Huber regression, and $\epsilon$-SVR primal problems
\begin{align*}
(P1) \quad \min_{h \in \mathcal{H}_\mathcal{K}}& ~~~ \frac{1}{2n}~\sum_{i=1}^n \| h(x_i) - y_i \|_{\mathcal{Y}, \epsilon}^2 + \frac{\Lambda}{2} \|h\|_{\mathcal{H}_\mathcal{K}}^2,\\
(P2) \quad \min_{h \in \mathcal{H}_\mathcal{K}}& ~~~~~ \frac{1}{n}~\sum_{i=1}^n \ell_{H, \kappa}(h(x_i) - y_i) + \frac{\Lambda}{2} \|h\|_{\mathcal{H}_\mathcal{K}}^2,\\
(P3) \quad \min_{h \in \mathcal{H}_\mathcal{K}}& ~~~~~ \frac{1}{n}~\sum_{i=1}^n \| h(x_i) - y_i \|_{\mathcal{Y}, \epsilon} + \frac{\Lambda}{2} \|h\|_{\mathcal{H}_\mathcal{K}}^2,
\end{align*}
are given by \Cref{eq:expansion}, with $\hat{\Omega} = \hat{W}V^{-1}$, and $\hat{W}$ the solution to the respective finite dimensional dual problems
\begin{equation*}
\begin{aligned}
(D1) \quad &\min_{W \in \mathbb{R}^{n \times n}} ~~ &&\frac{1}{2} \left\|AW - B\right\|_\mathrm{Fro}^2 + \epsilon~\|W \|_{2, 1},\\[0.3cm]
(D2) \quad &\min_{W \in \mathbb{R}^{n \times n}} ~~ &&\frac{1}{2} \left\|AW - B\right\|_\mathrm{Fro}^2,\\
&~~~\text{s.t.} && \|W\|_{2, \infty} \le \kappa,\\[0.3cm]
(D3) \quad &\min_{W \in \mathbb{R}^{n \times n}} ~~ &&\frac{1}{2} \left\|AW - B\right\|_\mathrm{Fro}^2 + \epsilon~\|W \|_{2, 1},\\
&~~~\text{s.t.} && \|W\|_{2, \infty} \le 1,
\end{aligned}
\end{equation*}
with $V$, $A$, $B$ such that: $VV^\top = K^Y$, $A^\top A = \frac{K^X}{\Lambda n} + \mathbf{I}_n$ (or $A^\top A = K^X/(\Lambda n)$ for the $\epsilon$-SVR), and $A^\top B = V$.
\end{theorem}

\Cref{thm:loss_instantiation}'s proof is detailed in \Cref{apx:loss_instantiation}.
If $\mathcal{K}$ is not identity decomposable, but only satisfies \Cref{hyp:sum_decompo}, the dual problems do not admit compact writings such as those of \Cref{thm:loss_instantiation}.
Nonetheless, they are still easily solvable, and the standard Ridge regression is recovered for $\epsilon=0$ or $\kappa = +\infty$.
This is discussed at length in the Appendix.

Problem $(D1)$ is a Multi-Task Lasso problem \citep{obozinski2010joint}.
It can be solved by Projected Gradient Descent (PGD), that involves the Block Soft Thresholding operator such that BST$(x, \tau) = \left(1 - \tau/\|x\|\right)_+x$.
Problem $(D2)$ is a constrained least square problem, that also admits a resolution through PGD, but with the Projection operator such that $\text{Proj}(x, \tau) = \min\left(\tau/\|x\|, 1\right)x$.
Finally, Problem $(D3)$ combines both non-smooth terms and consequently both projection steps.
Given a stepsize $\eta$, and $T$ a number of epoch, the algorithms are detailed in \Cref{alg:pgd}.
Note that $\tilde{K}$'s Singular Value Decomposition is not necessary, since the computations only involve $A^\top A = \widetilde{K}$ and $A^\top B = V$.

\begin{algorithm}[!ht]
\SetKwInOut{Input}{input}
\SetKwInOut{Init}{init}
\SetKwInOut{Parameter}{Param}
\caption{Projected Gradient Descents (PGDs)}
\Input{~Gram matrices $K^X$, $K^Y$, parameters $\Lambda$, $\epsilon$, $\kappa$}\vspace{0.1cm}
\Init{~$\widetilde{K} = \frac{1}{\Lambda n} K^X + \textbf{I}_n$ (or $\widetilde{K} = \frac{1}{\Lambda n} K^X$ for $\epsilon$-SVR),\\\vspace{0.1cm}~$K^Y = VV^\top$, $W = \bm{0}_{\mathbb{R}^{n \times n}}$}\vspace{0.1cm}
\For{epoch from $1$ to $T$}{\vspace{0.1cm}
    \tcp{gradient step}
    $W = W - \eta(\widetilde{K}W - V)$

    \tcp{projection step}
    \For{row $i$ from $1$ to $n$}{\vspace{0.1cm}

        $W_{i:} = \text{BST}\left(W_{i:}, \epsilon\right)$ \hspace{0.86cm}\tcp{if Ridge or SVR}

        $W_{i:} = \text{Proj}\left(W_{i:}, \kappa \text{~or~}1\right)$ \hspace{0.2cm}\tcp{if Huber or SVR}

        }
    }
    \vspace{0.1cm}

\Return{$W$}
\label{alg:pgd}
\end{algorithm}


\subsection{Approximate Dual Resolution with Huber Loss}

In this section we solve \Cref{pbm:dual_pbm_approx} for the Huber loss and $\mathcal{Y} = L^2[\Theta,\mu]$, with $\Theta$ a compact set endowed with measure $\mu$.
A classical choice of OVK is then $\mathcal{K}=k_{\mathcal{X}} \cdot T_k$, $k_{\mathcal{X}}$ being a scalar kernel over the inputs, and $T_k$ the integral operator associated to a scalar kernel $k \colon \Theta \times \Theta \to \mathbb{R}$ defined for all $g \in L^2[\Theta,\mu]$ by $T_k g = \int_{\Theta} k(\cdot,\theta) g(\theta) \mathrm{d}\mu(\theta)$.
Continuity of $k$ grants compactness of $T_k$, allowing for the methodology presented in \Cref{sec:dual_approx}.
In the following, $(\lambda_j,\psi_j)_{j=1}^m$ denotes the eigendecomposition of $T_k$, which is dependent both in $k$ and $\mu$, and can be obtained by solving a differential equation derived from the eigenvalue problem.
However, given that the optimal kernel $k$ is unknown, one can choose a Hilbertian basis $\{\psi_j\}_{j=1}^\infty$ of $L^2[\Theta,\mu]$ and a non-increasing summable sequence $(\lambda_j)_{j=1}^\infty \in \mathbb{R}_+^*$ to construct the kernel $k$, which gives direct access to $T_k$'s eigendecomposition.
\begin{theorem}\label{thm:huber_l2}
For an OVK $\mathcal{K} = k_{\mathcal{X}}~T_k$, an approximate solution to the Huber loss regression problem
\begin{equation*}
\min_{h \in \mathcal{H}_\mathcal{K}} ~ \frac{1}{n} \sum_{i=1}^n \ell_{H, \kappa}(h(x_i) - y_i) + \frac{\Lambda}{2} \|h\|_{\mathcal{H}_\mathcal{K}}^2,
\end{equation*}
is given by \Cref{eq:expansion_approx}, with $\hat{\Omega}$ the solution to the following constrained quadratic problem (with $R$ as in \Cref{hyp:FL_3}), that can be tackled by PGD in the spirit of \Cref{alg:pgd}:
\begin{problem}\label{pbm:huber_dual_l2}
\begin{aligned}
&\min_{\Omega \in \mathbb{R}^{n \times m}} ~~ && \mathbf{Tr}\left( \frac{1}{2} \Omega \Omega^\top + \frac{1}{2 \Lambda n} K^X \Omega S \Omega^\top -\Omega R^\top \right),\\
&~~~\text{s.t.} && \|\Omega\|_{2, \infty} \le \kappa.
\end{aligned}
\end{problem}
\end{theorem}

\begin{remark}
When $\kappa$ is large, one recovers the unconstrained Ridge regression problem, whose solution enjoys a closed form expression, and for which a resolution method based on an approximation of the inverse of the integral operator $T_k$ was presented in \citet{kadri2016ovk}.
\end{remark}


\subsection{Stability Analysis}

Algorithm stability is a notion introduced by \citet{bousquet2002stability}.
It links the \emph{stability} of an algorithm, \textit{i.e.} how removing a training observation impacts the algorithm output, to the algorithm generalization capacity, \textit{i.e.} how far the empirical risk of the algorithm output is to its true risk.
The rationale behind this approach is that standard analyses of Empirical Risk Minimization rely on a crude approximation consisting in bounding the empirical process $\sup_{h \in \mathcal{H}}|\hat{\mathcal{R}}_n(h) - \mathcal{R}(h)|$.
Indeed, considering a supremum over the whole hypothesis set seems very pessimistic, as decision functions with high discrepancy $|\hat{\mathcal{R}}_n(h) - \mathcal{R}(h)|$ would hopefully not be selected by the algorithm.
However, the limitation of stability approaches lies in that algorithms performances are never compared to an optimal solution~$h^*$.
Nevertheless, their capacity to deal with OVK machines without making the trace-class assumption (as opposed to Rademacher-based strategies, see \textit{e.g.} \citet{maurer2016bounds}) make them particularly well suited to our setting.
In the footsteps of \citet{audiffren2013stability}, we now derive stability bounds for our algorithms, which are all the more relevant as they make explicit the role of hyperparameters.
For any algorithm $A$, $h_{A(\sample)}$ and $h_{A(\samplei)}$ denote the decision functions output by the algorithm, respectively trained on samples $\sample$ and $\samplei = \sample \backslash \{(x_i, y_i)\}$.
Notice that symmetry among observations in \Cref{pbm:primal_pbm} cancels the impact of $i$.
Formally, algorithm stability states as follows.

\begin{definition}
\citep{bousquet2002stability}
Algorithm $A$ has stability $\beta$ if for any sample $\sample$, and any $i \le \#\sample$, it holds: $\sup_{(x,y)\in\mathcal{X}\times\mathcal{Y}}|\ell(h_{A(\sample)}(x), y) - \ell(h_{A(\samplei)}(x), y)|\le\beta$.
\end{definition}

\begin{assumption}\label{hyp:bounded_loss}
There exists $M>0$ such that for any \mbox{sample} $\sample$ and any realization $(x,y) \in \mathcal{X} \times \mathcal{Y}$ of $(X, Y)$ it holds: $\ell(h_{A(\sample)}(x), y) \le M$.
\end{assumption}

\begin{theorem}
\citep{bousquet2002stability}
Let $A$ be an algorithm with stability $\beta$ and loss function satisfying \mbox{\Cref{hyp:bounded_loss}}.
Then, for any $n\ge1$ and $\delta \in ]0,1[$ it holds with probability at least $1 - \delta$:\vspace{-0.3cm}
\begin{equation*}
\mathcal{R}(h_{A(\sample)}) \le \hat{\mathcal{R}}_n(h_{A(\sample)}) + 2\beta + (4n\beta + M)\sqrt{\frac{\ln(1/\delta)}{2n}}.
\end{equation*}
\end{theorem}

Stability for OVK machines such as in \Cref{pbm:primal_pbm} may be derived from the following two assumptions.

\begin{assumption}\label{hyp:bounded_kernel}
There exists $\gamma > 0$ such that for any input observation $x \in \mathcal{X}$ it holds: $\|\mathcal{K}(x, x)\|_\text{op} \le \gamma^2$.
\end{assumption}

\begin{assumption}\label{hyp:lipschitz_loss}
There exists $C > 0$ such that for any point
$(x,y) \in \mathcal{X} \times \mathcal{Y}$, any sample $\sample$, and any $i \le \#\sample$, it holds: $|\ell(h_\mathcal{S}(x), y) - \ell(h_{\mathcal{S}^{\backslash i}}(x), y)|\le C \|h_\mathcal{S}(x) - h_{\mathcal{S}^{\backslash i}}(x)\|_\mathcal{Y}$.
\end{assumption}

\begin{theorem}
\label{thm:stability_kadri}
\citep{audiffren2013stability}
If \Cref{hyp:bounded_kernel,hyp:lipschitz_loss} hold, then the algorithm returning the solution to \Cref{pbm:primal_pbm} has $\beta$ stability with $\beta \le C^2\gamma^2/(\Lambda n)$.
\end{theorem}

In order to get generalization bounds, we shall now derive constants $M$ and $C$ of \Cref{hyp:bounded_loss,hyp:lipschitz_loss} respectively.
This is usually done under the following assumption.

\begin{assumption}\label{hyp:bounded_outputs}
There exists $M_\mathcal{Y} > 0$ such that for any realization $y \in \mathcal{Y}$ of $Y$ it holds: $\|y\|_\mathcal{Y} \le M_\mathcal{Y}$.
\end{assumption}

\begin{remark}
It should be noticed that in structured prediction or structured data representation this assumption is directly fulfilled with $M_\mathcal{Y} = 1$.
Indeed, outputs (and potentially inputs) are actually some $y_i=\phi(z_i)$, with $\phi$ the canonical feature map associated to a scalar kernel, so that it suffices to choose a normalized kernel to satisfy \Cref{hyp:bounded_outputs}.
\end{remark}

\begin{theorem}\label{thm:stability}
Under \Cref{hyp:bounded_outputs}, algorithms previously described satisfy \Cref{hyp:bounded_loss,hyp:lipschitz_loss} with constants $M$ and $C$ as detailed in \Cref{tab:constants}.
\end{theorem}

\begin{figure*}[!t]
\begin{minipage}[b]{0.67\textwidth}
\begin{center}
\begin{footnotesize}
\begin{tabular}{c|c|c}
\toprule
                 & $M$ & $C$\\\midrule
$\epsilon$-SVR   & $\sqrt{M_\mathcal{Y} - \epsilon}\left(\frac{\sqrt{2}\gamma}{\sqrt{\Lambda}} + \sqrt{M_\mathcal{Y} - \epsilon}\right)$ & $1$\\[0.3cm]
$\epsilon$-Ridge & $(M_\mathcal{Y} - \epsilon)^2\left(1 + \frac{2\sqrt{2}\gamma}{\sqrt{\Lambda}} + \frac{2\gamma^2}{\Lambda}\right)$ & $2(M_\mathcal{Y} - \epsilon)\left(1 + \frac{\gamma\sqrt{2}}{\sqrt{\Lambda}}\right)$\\[0.3cm]
$\kappa$-Huber   & $\kappa\sqrt{M_\mathcal{Y} - \frac{\kappa}{2}}\left(\frac{\gamma\sqrt{2\kappa}}{\sqrt{\Lambda}} + \sqrt{M_\mathcal{Y} - \frac{\kappa}{2}}\right)$ & $\kappa$\\\bottomrule
\end{tabular}
\end{footnotesize}
\end{center}
\vspace{-0.3cm}
\caption{Algorithms Constants}
\label{tab:constants}
\end{minipage}
%
%
\begin{minipage}[b]{0.32\textwidth}
\begin{center}
\includegraphics[width=\textwidth]{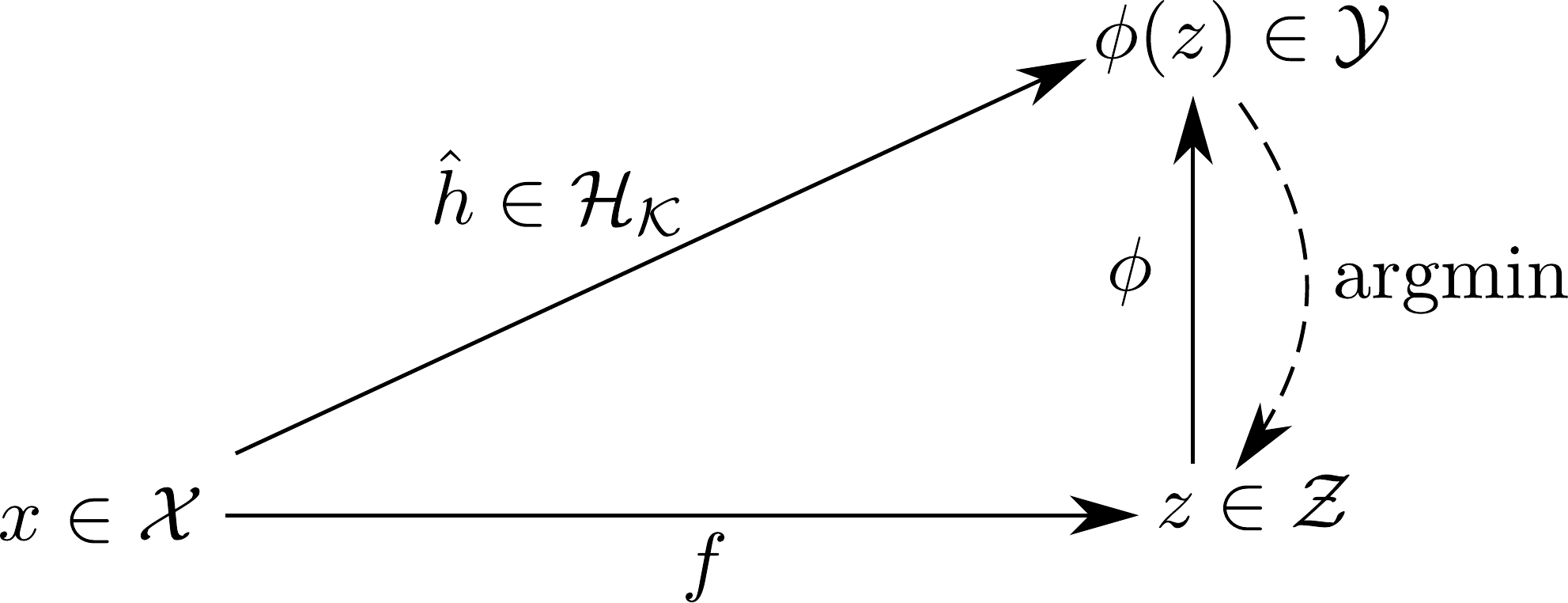}
\end{center}
\vspace{-0.3cm}
\caption{Output Kernel Regression}
\label{fig:okr}
\end{minipage}
\end{figure*}

\section{Applications and Numerical Experiments}
\label{sec:5-applis}


In this section, we discuss some applications unlocked by vv-RKHSs with infinite dimensional outputs.
In particular, structured prediction, structured representation learning, and functional regression are formally described, and numerical experiments highlight the benefits of the losses introduced.

\subsection{Application to Structured Output Prediction}
\label{sec:applis}


Assume one is interested in learning a predictive decision rule $f$ from a set $\mathcal{X}$ to a complex structured space $\mathcal{Z}$.
To bypass the absence of norm on $\mathcal{Z}$, one may design a (scalar) kernel $k$ on $\mathcal{Z}$, whose canonical feature map $\phi: z \mapsto k(\cdot, z)$ transforms any element of $\mathcal{Z}$ into an element of the (scalar) RKHS associated to $k$, denoted $\mathcal{Y}$ ($=\mathcal{H}_k$).
Learning a predictive model $f$ from $\mathcal{X}$ to $\mathcal{Z}$ boils down to learning a surrogate vector-valued model $h$ from $\mathcal{X}$ to $\mathcal{Y}$,
which is searched for in the vv-RKHS $ \mathcal{H}_\mathcal{K}$ associated to an OVK $\mathcal{K}$ by solving the following regularized empirical problem.
\begin{problem}\label{pbm:struct}
\hat{h} = \argmin_{h \in \mathcal{H}_\mathcal{K}} \frac{1}{n }\sum_{i=1}^n \ell(h(x_i), \phi(z_i)) + \frac{\Lambda}{2} \|h\|^2_{\mathcal{H}_\mathcal{K}}.
\end{problem}
%
Once $\hat{h}$ is learned, the predictions in $\mathcal{Z}$ are produced through a pre-image problem $f(x) = \argmin_{z \in \mathcal{Z}} \ell(\phi(z),\hat{h}(x))$.
This approach called Input Output Kernel Regression has been studied in several works \cite{Brouard_icml11,Kadri_icml2013}.
As an instance of the general Output Kernel Regression scheme of \Cref{fig:okr}, it belongs to the family of Surrogate Approaches for structured prediction (see \textit{e.g.} \citet{Ciliberto2016}).
While previous works have focused on identity decomposable kernels only, with the squared loss or hinge loss \citep{brouard2016input}, our general framework allows for many more losses.
The use of an $\epsilon$-insensitive loss in \Cref{pbm:struct}, in particular, seems adequate as it is a surrogate task, and inducing small mistakes that do not harm the inverse problem, while improving generalization, sounds as a suitable compromise.
We thus advocate to solve structured prediction in vv-RKHSs by using losses more sophisticated than the squared norm. In the following, the variants of IOKR are called accordingly to the loss they minimize: $\epsilon$-SV-IOKR, $\epsilon$-Ridge-IOKR, and Huber-IOKR.

{\bf YEAST dataset.} Although our approach's main strength of is to predict infinite dimensional outputs, we start with a simpler standard structured prediction dataset composed of $14$-dimensional outputs (the so-called YEAST dataset \citet{finley2008training}) described in the Supplements, on which comparisons and interpretations are easier.
We have collected results from \citet{finley2008training} and \citet{belanger2016structured}, and benchmarked our three algorithms.
Hyperparameters $\Lambda$, $\epsilon$, $\kappa$ have been selected among geometrical grids by cross-validation on the train dataset solely, and performances evaluated on the same test set as the above publications.
Results in terms of Hamming error are reported in \Cref{tab:yeast}, with significant improvements for the $\epsilon$-Ridge-IOKR and Huber-IOKR.
Furthermore, in order to highlight the interactions between our two ways of regularizing, \textit{i.e.} the RKHS norm and the $\epsilon$-insensitivity, we have plotted the $\epsilon$-Ridge-IOKR Mean Square Errors (the Hamming before clamping) and solution sparsity with respect to $\Lambda$ for $\epsilon$ varying from $1e$-$5$ to $1.5$ (\Cref{fig:mse_yeast,fig:sparse_yeast}):  $\Lambda$ and $\epsilon$ seem to act as competitive regularizations.
When $\Lambda$ is small, the regularization in $\epsilon$ is efficient, as solution with the best MSE is obtained for $\epsilon$ around $0.6$.
Conversely, when $\Lambda$ is big, no sparsity is induced, and having a high $\epsilon$ induces too much regularization.
Similar graphs for the $\epsilon$-SVR and $\kappa$-Huber are available in the Supplements, that highlight the superiority of the approaches for a wide range of hyperparameters.
A linear output kernel was used, such that solving the inverse problem boils down to clamping.

{\bf Metabolite dataset.} Regarding the infinite dimensional outputs, we have considered the metabolite identification problem \citep{Schymanski2016}, in which one aims at predicting molecules from their mass spectra.
For this task, Ridge-IOKR is the state-of-the-art approach, corresponding to our $\epsilon$-Ridge-IOKR with $\epsilon=0$.
Given the high number of constraints, Structured SVMs are not tractable as confirmed by our tests using the Pystruct lib implementation \cite{Mueller2014}.
This was already noticed in \citet{belanger2016structured} ($14$ is the maximum output dimension on which SSVMs were tested), and the implementation we tried indeed yielded very poor results despite prolonged training ($5\%,31\%,45\%$ top-$k$ errors).
We thus investigated the advantages of substituting the standard Ridge Regression for its $\epsilon$-insensitive version or a Huber regression.
Outputs (\textit{i.e.} metabolites) are embedded in an infinite dimensional Hilbert space through a Tanimoto-Gaussian kernel with $0.72$ bandwidth.
The dataset, presented in the Supplements and described at length in \citet{Brouard_ismb2016}, is composed of $6974$ mass spectra, while algorithms are compared through the top-$k$ accuracies, $k=1, 10, 20$.
Two $\Lambda$'s have been picked for their interesting behavior: one that yields the best performance for Ridge-IOKR, and the second that gives the best overall scores (hyperparameters $\epsilon$ and $\kappa$ being chosen to produce the best scores each time).
Again, results of \Cref{tab:metabolite} show improvements due to robust losses that are all the more important as the norm regularization is low, with an improvement on the best overall score.




\begin{table}[!h]
\caption{Top $1$~/~$10$~/~$20$ test accuracies (\%)}
\label{tab:metabolite}
\begin{center}
\begin{footnotesize}
\begin{sc}
\begin{tabular}{ccc}\toprule
$\Lambda$      & $1e$-$6$                                                  & $1e$-$4$\\\midrule
Ridge-IOKR           & 35.7 $\vert$ 79.9 $\vert$ 86.6                            & 38.1 $\vert$ 82.0 $\vert$ 88.9\\
$\epsilon$-Ridge-IOKR & 37.1 $\vert$ 81.7 $\vert$ 88.3                            & 36.3 $\vert$ 81.2 $\vert$ 87.9\\
Huber-IOKR & \textbf{38.3} $\vert$ \textbf{82.2} $\vert$ \textbf{89.1} & 37.7 $\vert$ 81.9 $\vert$ 88.8\\\bottomrule
\end{tabular}
\end{sc}
\end{footnotesize}
\end{center}
\end{table}

\begin{figure*}[!t]
%
\begin{minipage}[b]{0.3\textwidth}
\begin{center}
\includegraphics[width=\textwidth]{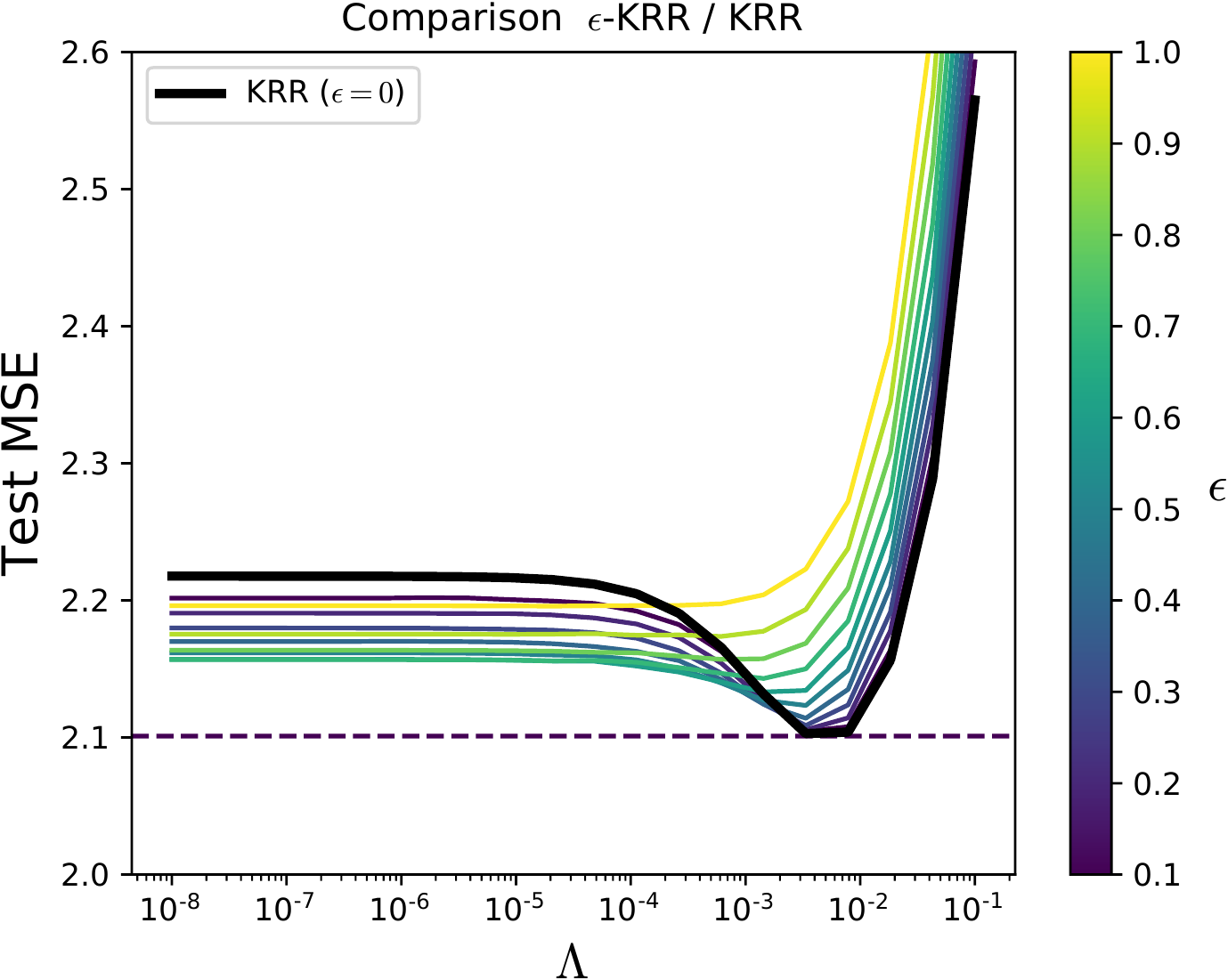}
\end{center}
\vspace{-0.3cm}
\caption{Test MSE w.r.t. $\Lambda$}
\label{fig:mse_yeast}
\end{minipage}
\hfill
\begin{minipage}[b]{0.3\textwidth}
\begin{center}
\includegraphics[width=\textwidth]{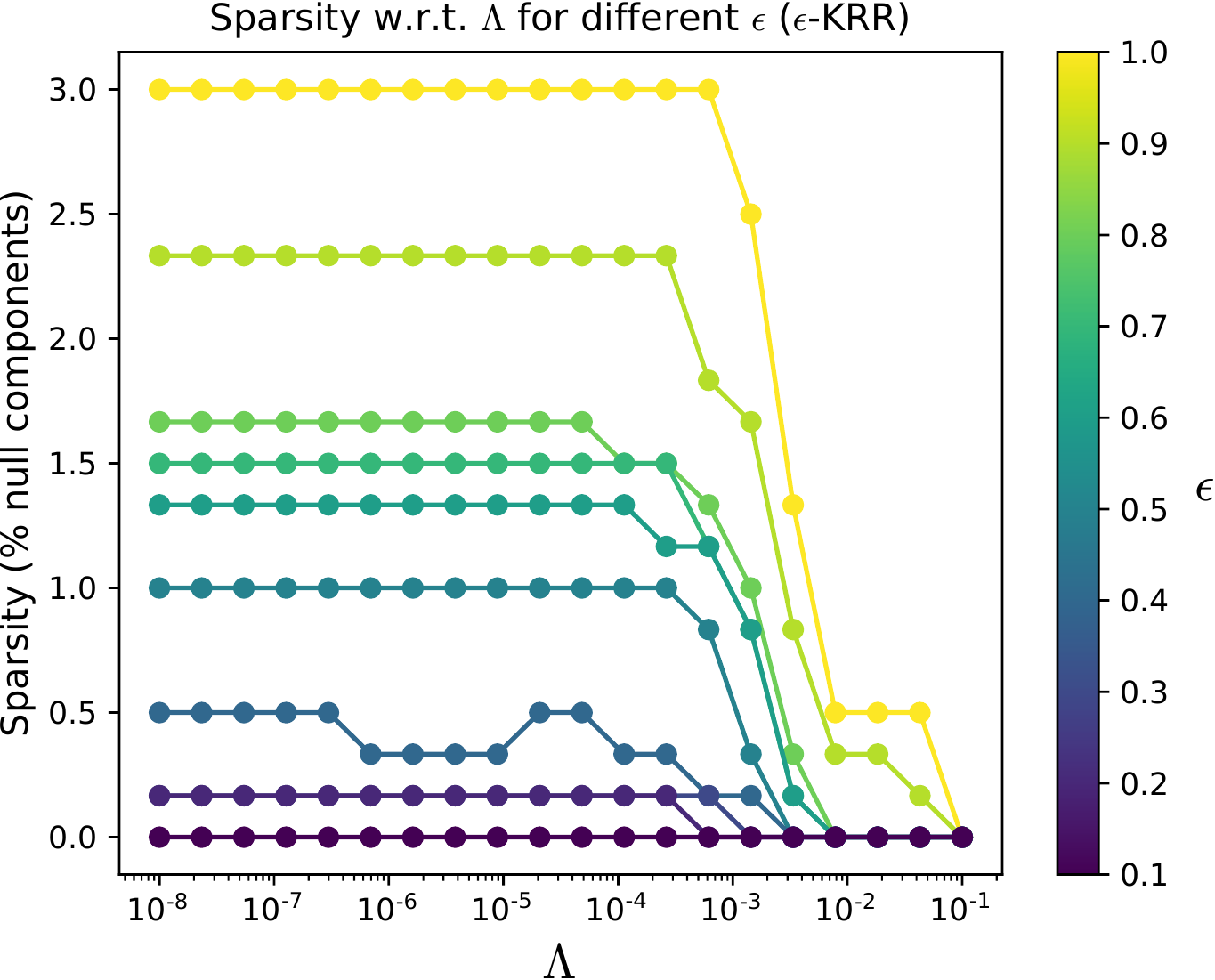}
\end{center}
\vspace{-0.3cm}
\caption{Sparsity w.r.t. $\Lambda$}
\label{fig:sparse_yeast}
\end{minipage}
\hfill
\begin{minipage}[b]{0.35\textwidth}
\begin{center}
\includegraphics[width=\textwidth, page=4]{fig/Autoencoder}
\end{center}
\vspace{-0.3cm}
\caption{$2$-Layer Kernel Autoencoder}
\label{fig:kae}
\end{minipage}
\end{figure*}

\begin{figure*}[!t]
\begin{minipage}[b]{0.25\textwidth}
\begin{center}
\begin{sc}
\begin{tabular}{cc}\toprule
SSVM                  & 20.2\\
SPEN                  & 20.0\\\midrule
$\epsilon$-Ridge-IOKR & \textbf{19.0}\\
Huber-IOKR            & 19.1\\
$\epsilon$-SV-IOKR    & 21.1\\\bottomrule
\end{tabular}
\end{sc}
\end{center}
\vspace{0.6cm}
\caption{YEAST Hamming errors}
\label{tab:yeast}
\end{minipage}
\hfill
\begin{minipage}[b]{0.38\textwidth}
\begin{center}
\includegraphics[width=\textwidth]{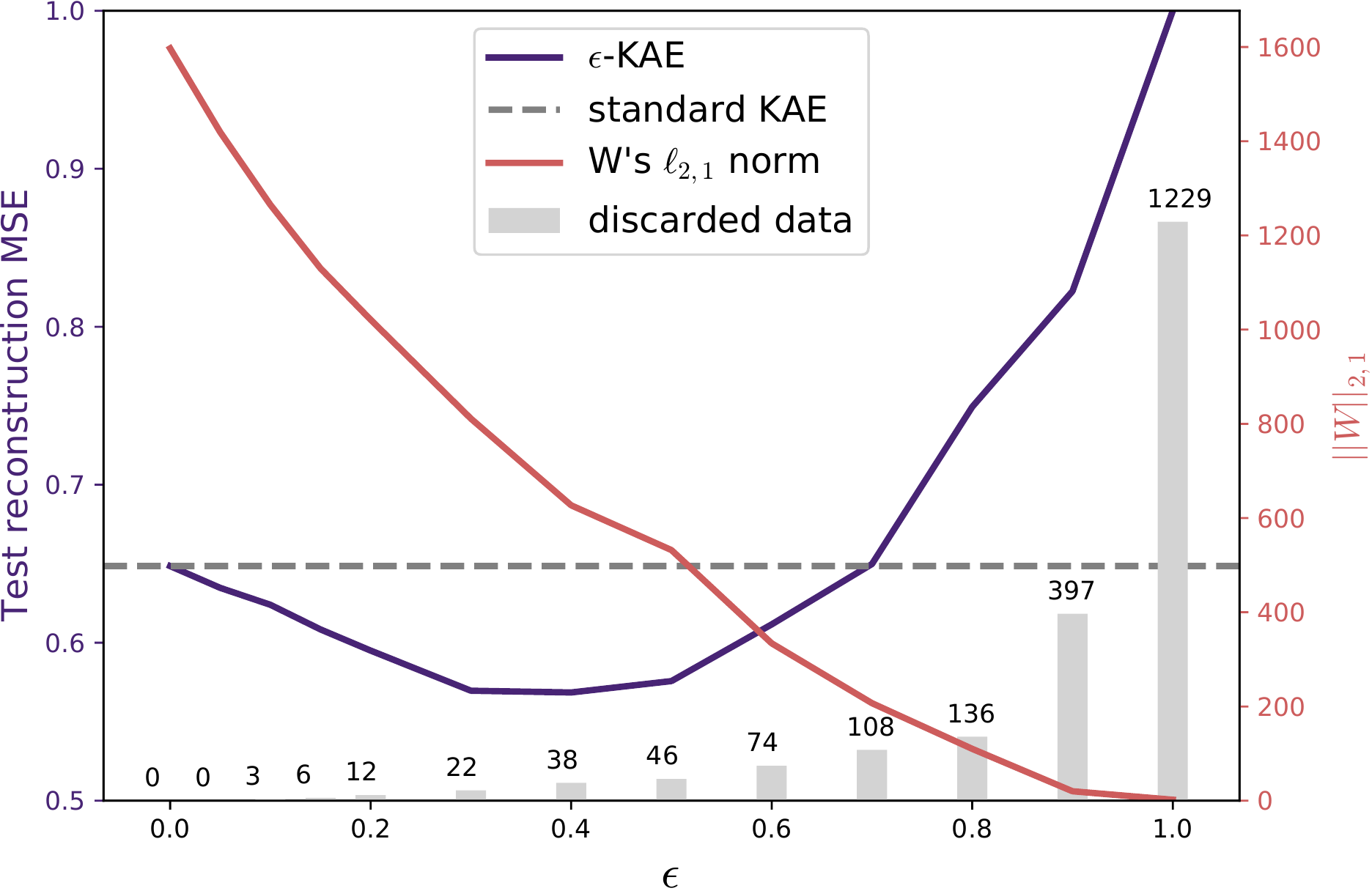}
\end{center}
\vspace{-0.3cm}
\caption{Reconstruction error w.r.t. $\epsilon$}
\label{fig:eps_kae}
\end{minipage}
\hfill
\begin{minipage}[b]{0.33\textwidth}
\begin{center}
\includegraphics[width=\textwidth]{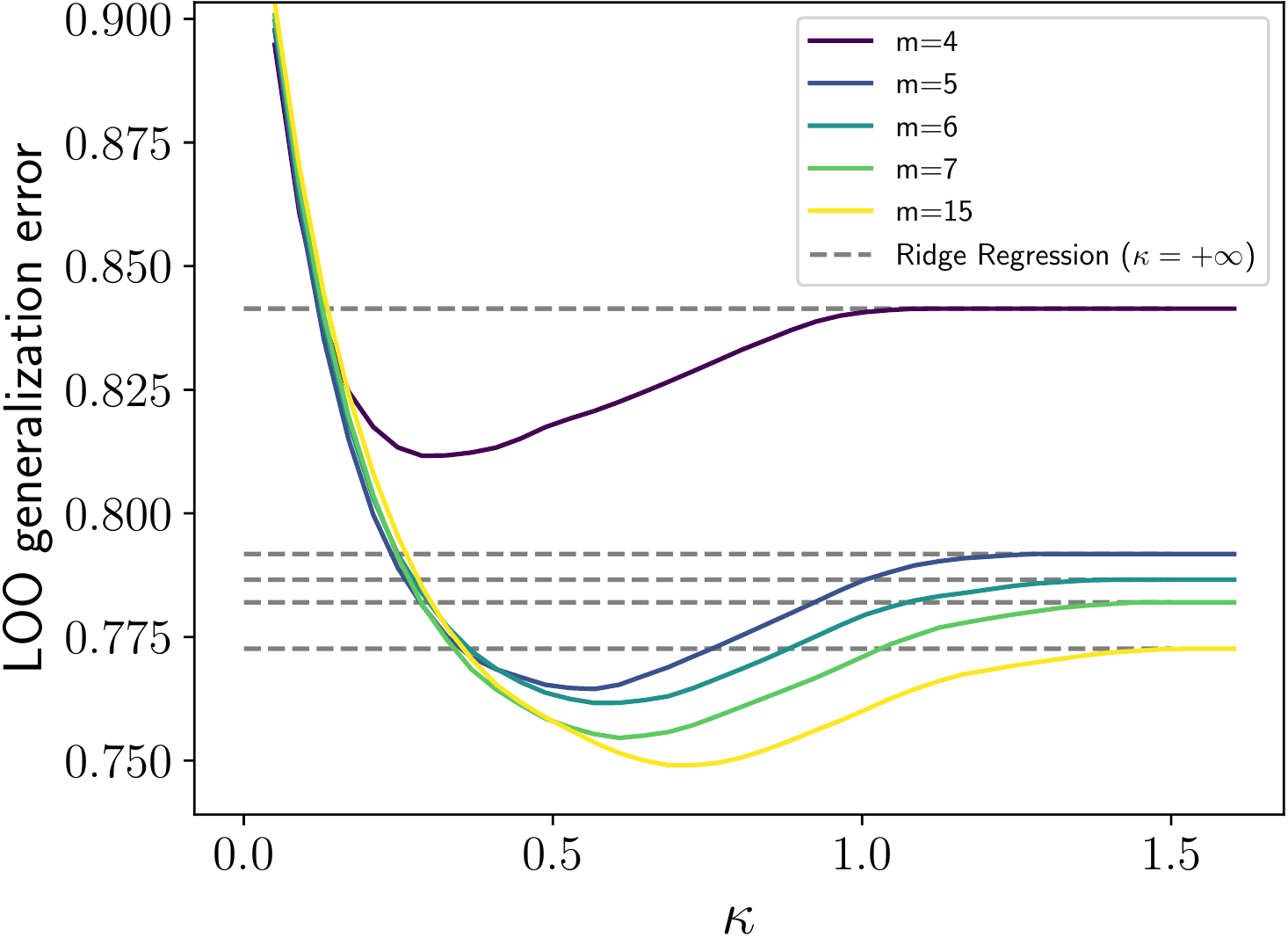}
\end{center}
\vspace{-0.3cm}
\caption{LOO error w.r.t. $\kappa$}
\label{fig:huber_fonctional}
\end{minipage}
\end{figure*}


\subsection{Structured Representation Learning}
Extracting vectorial representations from structured inputs is another task that can be tackled in vv-RKHSs \citep{laforgue2019autoencoding}.
This is a relevant approach in many cases: when complex data are uniquely available under the form of a similarity matrix for instance, for preserving privacy, or when deep neural networks fail to tackle structured objects as raw data.
Embedding data into a Hilbert space makes sense.
Then, composing functions in vv-RKHSs results in a Kernel Autoencoder (KAE, \Cref{fig:kae}) that outputs finite codes by minimizing the (regularized) discrepancy:
\begin{problem}\label{pbm:kae}
\frac{1}{2n} \sum_{i=1}^n \|\phi(x_i) - f_2\circ f_1(\phi(x_i))\|_\mathcal{Y}^2 + \Lambda~\text{Reg}(f_1, f_2).
\end{problem}
Again, this reconstruction loss is not the real goal, but rather a proxy to make the internal representation meaningful.
Therefore, all incentives to use $\epsilon$-insensitive losses or the Huber loss still apply.
The inferred $\epsilon$-KAE and Huber-KAE, obtained by changing the loss function in \Cref{pbm:kae}, are optimized as follows: the first layer coefficients are updated by Gradient Descent, while the second ones are reparametrized into $W_2$ and updated through PGD (instead of KRR closed form for standard KAEs).
This has been applied to a drug dataset, introduced in \citet{su2010structured} as an extract from the NCI-Cancer database.
As shown in \Cref{fig:eps_kae}, the $\epsilon$-insensitivity improves the generalization while inducing sparsity.
The $\epsilon$-insensitive framework is thus particularly promising in the context of Autoencoders.

\subsection{Function-to-Function Regression}

Regression with both inputs and outputs of functional \mbox{nature} is a challenging problem at the crossroads of Functional Data Analysis \cite{ramsay2007applied} and Machine Learning \cite{kadri2016ovk}.
While Functional Linear Modeling is the most common approach to address function-to-function regression, nonparametric approaches based on vv-RKHSs have emerged, that rely on the minimization of a squared loss.
However, robustness to abnormal functions is particularly meaningful in a field where data come from sensors and are used to monitor physical assets.
To the best of our knowledge, robust regression has only been tackled in the context of Functional Linear Models \cite{Kalogridis2019}.
We propose here to highlight the relevance of OVK machines learned with a Huber loss by solving \Cref{pbm:huber_dual_l2} for various levels $\kappa$.

{\bf Lip acceleration from EMG dataset.}
We consider the problem of predicting lip acceleration among time from electromyography (EMG) signals \citep{ramsay2007applied}.
The dataset consists of 32 records of the lower lip trajectory over 641 timestamps, and the associated EMG records, augmented with 4 outliers to assess the robustness of our approach.
Usefulness of minimizing the Huber loss is illustrated in \Cref{fig:huber_fonctional} by computing the Leave-One-Out (LOO) error associated to each model for various values of $m$.
For each $m$, as $\kappa$ grows larger than a threshold, the constraint on $\|\Omega\|_{2, \infty}$ becomes void and we recover the Ridge Regression solution.
The kernel chosen is given by $k_{\mathcal{X}}(x_1,x_2) = \int_{0}^1 \exp{(|x_1(\theta)-x_2(\theta)|)} \mathrm{d}\theta$, with $(\psi_j)_{j=1}^m$ being the harmonic basis in sine and cosine of $L^2[0,1]$, and $(\lambda_j)_{j=1}^m = (1/(j+1)^2)_{j=1}^m$.

\subsection{Related Work}

Another application of the presented results, both theoretical and computational, is the generalization of the \textit{loss trick}, see \textit{e.g.} \citet{Ciliberto2016}.
In the context of Output Kernel Regression, the latter stipulates that for suitable losses, the decoding expresses in terms of loss evaluations.
The work by \citet{luise2019leveraging} has extended this trick to penalization schemes different from the natural vv-RKHS norm.
Our findings, and the double expansion in particular, suggest that the loss trick can still be used with other surrogate loss functions than the squared norm, opening the door to a wide range of applications.

\section{Conclusion}\label{sec:conclu}

This work presents a versatile framework based on duality to learn OVK machines with infinite dimensional outputs.
The case of convolved losses (\textit{e.g.} $\epsilon$-insensitive, Huber) is thoroughly tackled, from algorithmic procedures to stability analysis.
This offers novel ways to enforce sparsity and robustness when learning within vv-RKHSs, opening an avenue for new applications on structured and functional data (\textit{e.g.} anomaly detection, robust prediction).
Future research directions could feature a calibration study of these novel surrogate approaches, or the introduction of kernel approximations such as random Fourier features, that would benefit our framework twice: both in input and in output. \par

\paragraph{Acknowledgment.} This work has been partially funded by the industrial chair \href{https://datascienceandai.wp.imt.fr/}{\textit{Data Science \& Artificial Intelligence for Digitalized Industry and Services}} from T\'{e}l\'{e}com Paris.
Authors would like to thank Olivier Fercoq for his insightful discussions and helpful comments.

\bibliography{ref}

\appendix
\onecolumn

The Supplementary Material is organized as follows.
\Cref{apx:proofs} collects the technical proofs of the core article's results.
\Cref{apx:loss_fig} provides illustrations of the main loss functions considered ($\epsilon$-insensitive Ridge and SVR, $\kappa$-Huber) in $1$ and $2$ dimensions.
\Cref{apx:expes} gathers additional details about the experimental protocols and the code furnished.

\section{Technical Proofs}
\label{apx:proofs}


\subsection{Proof of \Cref{thm:dual_pbm}}
\label{sec:apx_dual}

First, notice that the primal problem
\begin{equation*}
\min_{h \in \mathcal{H}_\mathcal{K}} \quad \frac{1}{n} \sum_{i=1}^n \ell(h(x_i), y_i) + \frac{\Lambda}{2} \|h\|_{\mathcal{H}_\mathcal{K}}^2
\end{equation*}
can be rewritten
\begin{align*}
\min_{h \in \mathcal{H}_\mathcal{K}} \quad & \sum_{i=1}^n \ell_i(u_i) + \frac{\Lambda n}{2} \|h\|_{\mathcal{H}_\mathcal{K}}^2,\\
\text{s.t.} \quad & u_i = h(x_i) \quad \forall i \le n.
\end{align*}
Therefore, with the notation $\bm{u} = (u_i)_{i \le n}$ and $\bm{\alpha} = (\alpha_i)_{i \le n}$, the Lagrangian writes
\begin{align*}
\mathscr{L}(h, \bm{u}, \bm{\alpha}) &= \sum_{i=1}^n \ell_i(u_i) + \frac{\Lambda n}{2} \|h\|_{\mathcal{H}_\mathcal{K}}^2 + \sum_{i=1}^n \left\langle \alpha_i, u_i - h(x_i) \right\rangle_\mathcal{Y},\\
&= \sum_{i=1}^n \ell_i(u_i) + \frac{\Lambda n}{2} \|h\|_{\mathcal{H}_\mathcal{K}}^2 + \sum_{i=1}^n \left\langle \alpha_i, u_i\right\rangle_\mathcal{Y} - \sum_{i=1}^n \left\langle \mathcal{K}(\cdot, x_i)\alpha_i, h\right\rangle_{\mathcal{H}_\mathcal{K}}.
\end{align*}
Differentiating with respect to $h$ and using the definition of the Fenchel-Legendre transform, one gets
\begin{align*}
g(\bm{\alpha}) &= \inf_{h \in \mathcal{H}_\mathcal{K}, \bm{u} \in \mathcal{Y}^n} \mathscr{L}(h, \bm{u}, \bm{\alpha}),\\
&= \sum_{i=1}^n~\inf_{u_i \in \mathcal{Y}} \left\{\ell_i(u_i) + \left\langle \alpha_i, u_i\right\rangle_\mathcal{Y}\right\} + \inf_{h \in \mathcal{H}_\mathcal{K}} \left\{\frac{\Lambda n}{2} \|h\|_{\mathcal{H}_\mathcal{K}}^2 - \sum_{i=1}^n \left\langle \mathcal{K}(\cdot, x_i)\alpha_i, h\right\rangle_{\mathcal{H}_\mathcal{K}}\right\},\\
&= \sum_{i=1}^n -\ell_i^\star(-\alpha_i) - \frac{1}{2 \Lambda n} \sum_{i, j=1}^n \left\langle \alpha_i, \mathcal{K}(x_i, x_j) \alpha_j\right\rangle_\mathcal{Y},
\end{align*}
together with the equality $\displaystyle \hat{h} = \frac{1}{\Lambda n} \sum_{i=1}^n \mathcal{K}(\cdot, x_i) \alpha_i$.
The conclusion follows immediately.\qed


\subsection{Proof of \Cref{thm:double_rt}}
\label{apx:double_rt}

As a reminder, our goal is to compute the solutions to the following problem:
\begin{align*}
\hat{h} \in \argmin_{h \in \mathcal{H}_\mathcal{K}} ~ \frac{1}{n} \sum_{i=1}^n \ell(h(x_i), y_i) + \frac{\Lambda}{2} \|h\|_{\mathcal{H}_\mathcal{K}}^2.
\end{align*}
Using \Cref{thm:dual_pbm}, one gets that $\hat{h} = \frac{1}{\Lambda n} \sum_{i=1}^n \mathcal{K}(\cdot, x_i) \hat{\alpha}_i$, with the $(\hat{\alpha}_i)_{i \le n}$ satisfying:
\begin{align*}
(\hat{\alpha}_i)_{i=1}^n \in \argmin_{(\alpha_i)_{i=1}^n \in \mathcal{Y}^n} ~ \sum_{i=1}^n\ell_i^\star(-\alpha_i) + \frac{1}{2\Lambda n} \sum_{i,j=1}^n \left\langle \alpha_i, \mathcal{K}(x_i, x_j) \alpha_j\right\rangle_\mathcal{Y}.
\end{align*}
However, this optimization problem cannot be solved in a straightforward manner, as $\mathcal{Y}$ is in general infinite dimensional.
Nevertheless, it is possible to bypass this difficulty by noticing that the optimal $(\hat{\alpha}_i)_{i \le n}$ actually lie in $\bm{\mathsf{Y}}^n$.
To show this, we decompose each coefficient as $\hat{\alpha}_i = \alpha_i^{\bm{\mathsf{Y}}} + \alpha_i^\perp$, with $(\alpha_i^{\bm{\mathsf{Y}}})_{i \le n}, (\alpha_i^\perp)_{i \le n} \in \bm{\mathsf{Y}}^n \times {\bm{\mathsf{Y}}^\perp}^n$.
Then, noticing that non-null $(\alpha_i^\perp)_{i \le n}$ necessarily increase the objective, we can conclude that the optimal $(\hat{\alpha}_i)_{i \le n}$ have no components among $\bm{\mathsf{Y}}^\perp$, or equivalently pertain to $\bm{\mathsf{Y}}$.
Indeed, by virtue of \Cref{hyp:FL_1,hyp:stable_ovk}, it holds:

\begin{equation*}
\sum_{i=1}^n\ell_i^\star(-\alpha_i^{\bm{\mathsf{Y}}}) + \frac{1}{2\Lambda n} \sum_{i,j=1}^n \left\langle \alpha_i^{\bm{\mathsf{Y}}}, \mathcal{K}(x_i, x_j) \alpha_j^{\bm{\mathsf{Y}}}\right\rangle_\mathcal{Y} \le \sum_{i=1}^n\ell_i^\star(-\alpha_i^{\bm{\mathsf{Y}}} - \alpha_i^\perp) + \frac{1}{2\Lambda n} \sum_{i,j=1}^n \left\langle \alpha_i^{\bm{\mathsf{Y}}} + \alpha_i^\perp, \mathcal{K}(x_i, x_j)(\alpha_j^{\bm{\mathsf{Y}}} + \alpha_j^\perp)\right\rangle_\mathcal{Y}.
\end{equation*}
If the inequality about $\ell_i^\star$ follows directly \Cref{hyp:FL_1}, that about $\mathcal{K}(x_i, x_j)$ can be obtained by \Cref{hyp:stable_ovk} as follows:
\begin{align*}
\sum_{i,j=1}^n \big\langle \alpha_i^{\bm{\mathsf{Y}}} + \alpha_i^\perp,& \mathcal{K}(x_i, x_j)(\alpha_j^{\bm{\mathsf{Y}}} + \alpha_j^\perp)\big\rangle_\mathcal{Y}\\
&= \sum_{i,j=1}^n \left\langle \alpha_i^{\bm{\mathsf{Y}}}, \mathcal{K}(x_i, x_j)\alpha_j^{\bm{\mathsf{Y}}})\right\rangle_\mathcal{Y} + 2 \sum_{i,j=1}^n \left\langle \alpha_i^\perp, \mathcal{K}(x_i, x_j)\alpha_j^{\bm{\mathsf{Y}}}\right\rangle_\mathcal{Y} + \sum_{i,j=1}^n \left\langle\alpha_i^\perp, \mathcal{K}(x_i, x_j)\alpha_j^\perp\right\rangle_\mathcal{Y},\\
&= \sum_{i,j=1}^n \left\langle \alpha_i^{\bm{\mathsf{Y}}}, \mathcal{K}(x_i, x_j)\alpha_j^{\bm{\mathsf{Y}}})\right\rangle_\mathcal{Y} + \sum_{i,j=1}^n \left\langle\alpha_i^\perp, \mathcal{K}(x_i, x_j)\alpha_j^\perp\right\rangle_\mathcal{Y},\\
&\ge \sum_{i,j=1}^n \left\langle \alpha_i^{\bm{\mathsf{Y}}}, \mathcal{K}(x_i, x_j)\alpha_j^{\bm{\mathsf{Y}}})\right\rangle_\mathcal{Y},
\end{align*}
where we have used successively \Cref{hyp:stable_ovk} and the positiveness of $\mathcal{K}$.
So there exists $\Omega = [\omega_{ij}]_{1 \le i, j \le n} \in \mathbb{R}^{n \times n}$ such that for all $i \le n$, $\hat{\alpha}_i = \sum_j \omega_{ij}~y_j$.
This proof technique is very similar in spirit to that of the Representer Theorem, and yields an analogous result, the reduction of the search space to a smaller vector space, as discussed at length in the main text.
The dual optimization problem thus rewrites:
\begin{align}
\min_{\Omega \in \mathbb{R}^{n \times n}} ~ &\sum_{i=1}^n\ell_i^\star\left(- \sum_{j=1}^n \omega_{ij}~y_j\right) + \frac{1}{2\Lambda n} \sum_{i,j=1}^n \left\langle \sum_{k=1}^n \omega_{ik}~y_k, \mathcal{K}(x_i, x_j) \sum_{l=1}^n \omega_{jl}~y_l\right\rangle_\mathcal{Y}\nonumber\\
\min_{\Omega \in \mathbb{R}^{n \times n}} ~ &\sum_{i=1}^n L_i\left( (\omega_{ij})_{j \le n}, K^Y \right) + \frac{1}{2\Lambda n} \sum_{i,j,k,l=1}^n \omega_{ik}~\omega_{jl}~\left\langle y_k, \sum_{t=1}^T k_t(x_i, x_j) A_t y_l\right\rangle_\mathcal{Y},\nonumber\\
\min_{\Omega \in \mathbb{R}^{n \times n}} ~ &\sum_{i=1}^n L_i\left( \Omega_{i:}, K^Y \right) + \frac{1}{2 \Lambda n} \mathbf{Tr}\left(\tilde{M}^\top (\Omega \otimes \Omega)\right),\label{eq:any_kernel}
\end{align}
with $M$ the $n \times n \times n \times n$ tensor such that $M_{ijkl} = \langle y_k , \mathcal{K}(x_i, x_j)y_l\rangle_\mathcal{Y}$, and $\tilde{M}$ its rewriting as a $n^2 \times n^2$ block matrix such that its $(i, j)$ block is the $n \times n$ matrix with elements $\tilde{M}^{(i,j)}_{st} = \left\langle y_j, \mathcal{K}(x_i, x_s)y_t\right\rangle_\mathcal{Y}$.

The second term is quadratic in $\Omega$, and consequently convex.
As for the $L_i$'s, they are basically rewritings of the Fenchel-Legendre transforms $\ell_i^\star$'s that ensure the computability of the problem (they only depend on $K^Y$, which is known).
Regarding their convexity, we know by definition that the $\ell_i^\star$'s are convex.
Composing by a linear function preserving the convexity, we know that each $L_i$ is convex with respect to $\Omega_{i:}$, and therefore with respect to $\Omega$.

Thus, we have first converted the infinite dimensional primal problem in $\mathcal{H}_\mathcal{K}$ into an infinite dimensional dual problem in $\mathcal{Y}^n$, which in turn is reduced to a convex optimization procedure over $\mathbb{R}^{n \times n}$, that only involves computable quantities.

If $\mathcal{K}$ satisfies \Cref{hyp:sum_decompo}, the tensor $M$ simplifies to
\begin{equation*}
M_{ijkl} = \left\langle y_k , \mathcal{K}(x_i, x_j)y_l\right\rangle_\mathcal{Y} = \sum_{t=1}^T k_t(x_i, x_j) \left\langle y_k, A_t y_l\right\rangle_\mathcal{Y} = \sum_{t=1}^T [K^X_t]_{ij} [K^Y_t]_{kl},
\end{equation*}
and the problem rewrites
\begin{equation*}
\min_{\Omega \in \mathbb{R}^{n \times n}} ~ \sum_{i=1}^n L_i\left( \Omega_{i:}, K^Y \right) + \frac{1}{2\Lambda n} \sum_{t=1}^T \mathbf{Tr}\left( K_t^X \Omega K_t^Y \Omega^\top\right).
\end{equation*}
\qed

\begin{remark}
The second term of Problem \eqref{eq:any_kernel} can be easily optimized.
Indeed, let $\tilde{M}$ be a block matrix such that $\tilde{M}^{(i,j)}_{st} = \tilde{M}^{(s,t)}_{ij}$ for all $i,j,s,t \le n$.
Notice that $\tilde{M}$ as defined earlier satisfies this condition as a direct consequence of the OVK symmetry property.
Then it holds:

\begin{equation*}
\frac{\partial\mathbf{Tr}\left(\tilde{M}^\top (\Omega \otimes \Omega)\right)}{\partial\omega_{st}} = 2 \mathbf{Tr}\left(\tilde{M}^{(s,t) \top} \Omega\right).
\end{equation*}
\end{remark}

Indeed, notice that $\mathbf{Tr}\left(\tilde{M}^\top (\Omega \otimes \Omega)\right) = \sum_{i,j=1}^n \omega_{ij} \mathbf{Tr}\left(\tilde{M}^{(i,j) \top} \Omega\right)$ and use the symmetry assumption.
In the particular case of a decomposable kernel, it holds that $\tilde{M}^{(i,j)} = K^X_{i:} {K^Y_{j:}}^\top$ so that
\begin{equation*}
\frac{\partial\mathbf{Tr}\left(\tilde{M}^\top (\Omega \otimes \Omega)\right)}{\partial\omega_{st}} = 2 \mathbf{Tr}\left(\tilde{M}^{(s,t) \top} \Omega\right) = 2 \sum_{i,j=1}^n \left[K^X_{s:} {K^Y_{t:}}^\top\right]_{ij} \omega_{ij} = 2\sum_{ij=1}^n K^X_{si} K^Y_{tj} \omega_{ij} = 2\left[K^X\Omega K^Y\right]_{st},
\end{equation*}
and one recovers the gradients established in \Cref{eq:gradient}.


\subsection{Proof of \Cref{prop:good_losses}}\label{apx:good_losses}

The proof technique is the same for all losses: first explicit the FL transforms $\ell_i^\star$, then use simple arguments to verify Assumptions \ref{hyp:FL_1} and \ref{hyp:FL_2}.
For instance, any increasing function of $\|\alpha\|$ automatically satisfy the assumptions.

\begin{itemize}
\item Assume that $\ell$ is such that there is $f:\mathbb{R} \rightarrow \mathbb{R}$ convex, $\forall i \le n, \exists z_i \in Y,~\ell_i(y) = f(\left\langle y, z_i\right\rangle)$.
Then $\ell_i^\star: \mathcal{Y} \rightarrow \mathbb{R}$ writes $\ell_i^\star(\alpha) = \sup_{y \in \mathcal{Y}} \left\langle \alpha, y\right\rangle - f(\left\langle y, z_i\right\rangle)$.
If $\alpha$ is not collinear to $z_i$, this quantity is obviously $+ \infty$.
Otherwise, assume that $\alpha = \lambda z_i$.
The FL transform rewrites: $\ell_i^\star(\alpha) = \sup_t \lambda t - f(t) = f^\star(\lambda) = f^\star(\pm \|\alpha\|/\|z_i\|)$.
Finally, $\ell_i^\star(\alpha) = \chi_{\text{span}(z_i)}(\alpha) + f^\star\left(\pm \frac{\|\alpha\|}{\|z_i\|}\right)$.
If $\alpha \notin Y$, then \textit{a fortiori} $\alpha \notin \text{span}(z_i)$, so $\ell_i^\star(\alpha^Y + \alpha^\perp) = +\infty \ge \ell_i^\star(\alpha^Y)$ for all $(\alpha^Y, \alpha^\perp) \in Y \times Y^\perp$.
For all $i \le n$, $\ell_i^\star$ satisfy \Cref{hyp:FL_1}.
As for \Cref{hyp:FL_2}, if $\alpha = \sum_{i=1}^n c_i y_i$, then $\chi_{\text{span}(z_i)}(\alpha)$ only depends on the $(c_i)_{i \le n}$
Indeed, assume that $z_i \in Y$ writes $\sum_j b_j y_j$.
Then $\chi_{\text{span}(z_i)}(\alpha)$ is equal to $0$ if there exists $\lambda \in \mathbb{R}$ such that $c_j = \lambda b_j$ for all $j \le n$, and to $+\infty$ otherwise.
The second term of $\ell_i^\star$ depending only on $\|\alpha\|$, it directly satisfies \Cref{hyp:FL_2}.
This concludes the proof.
\item Assume that $\ell$ is such that there is $f: \mathbb{R}_+ \rightarrow \mathbb{R}$ convex increasing, with $\frac{f'(t)}{t}$ continuous over $\mathbb{R}_+$, $\ell(y) = f(\|y\|)$.
Although this loss may seem useless at the first sight since $\ell$ does not depend on $y_i$, it should not be forgotten that the composition with $y \mapsto y - y_i$ does not affect the validation of Assumptions \ref{hyp:FL_1} and \ref{hyp:FL_2} (see below).
One has: $\ell^\star(\alpha) = \sup_{y \in \mathcal{Y}} \left\langle \alpha, y\right\rangle - f(\|y\|)$.
Differentiating w.r.t. $y$, one gets: $\alpha = \frac{f'(\|y\|)}{\|y\|}y$, which is always well define as $t \mapsto \frac{f'(t)}{t}$ is continuous over $\mathbb{R}_+$.
Reverting the equality, it holds: $y = \frac{{f'}^{-1}(\|\alpha\|)}{\|\alpha\|} \alpha$, and $\ell^\star(\alpha) = \|\alpha\|{f'}^{-1}(\|\alpha\|) - f\circ{f'}^{-1}(\|\alpha\|)$.
This expression depending only on $\|\alpha\|$, \Cref{hyp:FL_2} is automatically satisfied.
Let us now investigate the monotonicity of $\ell^\star$ w.r.t. $\|\alpha\|$.
Let $g: \mathbb{R}_+ \rightarrow \mathbb{R}$ such that $g(t) = t{f'}^{-1}(t) - f\circ{f'}^{-1}(t)$.
Then $g'(t) = {f'}^{-1}(t) \ge 0$.
Indeed, as $f':\mathbb{R}_+ \rightarrow \mathbb{R}_+$ is always positive due to the monotonicity of $f$, so is ${f'}^{-1}$.
This final remark guarantees that $\ell^\star$ is increasing with $\|\alpha\|$.
It is then direct that $\ell^\star$ fulfills \Cref{hyp:FL_1}.
\item Assume that $\ell(y) = \lambda \|y\|$.
It holds $\ell^\star(\alpha) = \chi_{\mathcal{B}_\lambda}(\alpha)$.
So $\ell^\star$ is increasing w.r.t. $\|\alpha\|$: it fulfills Assumptions \ref{hyp:FL_1} and \ref{hyp:FL_2}.
\item Assume that $\ell(y) = \chi_{\mathcal{B}_\lambda}(y)$.
It holds $\ell^\star(\alpha) = \lambda \|\alpha\|$.
The monotonicity argument also applies.
\item Assume that $\ell(y) = \lambda \|y\|\log(\|y\|)$.
It can be shown that $\ell^\star(\alpha) = \lambda e^{\frac{\|\alpha\|}{\lambda} - 1}$.
The same argument as above applies.
\item Assume that $\ell(y) = \lambda (\exp(\|y\|) - 1)$.
It can be shown that $\ell^\star(\alpha) = \mathbb{I}\{\|\alpha\| \ge \lambda\}\cdot \left(\|\alpha\|\log\left(\frac{\|\alpha\|}{\lambda e}\right) + \lambda\right)$.
Again, the FL transform is an increasing function of $\|\alpha\|$: it satisfies Assumptions \ref{hyp:FL_1} and \ref{hyp:FL_2}.
\item Assume that $\ell_i(y) = f(y - y_i)$, with f such that $f^\star$ fulfills Assumptions \ref{hyp:FL_1} and \ref{hyp:FL_2}.
Then $\ell_i^\star(\alpha) = \sup_{y \in \mathcal{Y}} \left\langle \alpha, y\right\rangle - f(y - y_i) = f^\star(\alpha) + \left\langle \alpha, y_i\right\rangle$.
If $f^\star$ satisfies Assumptions \ref{hyp:FL_1} and \ref{hyp:FL_2}, then so does $\ell_i^\star$.
This remark is very important, as it gives more sense to loss function based on $\|y\|$ only, since they can be applied to $y - y_i$ now.
\item Assume that there exists $f, g$ satisfying Assumptions \ref{hyp:FL_1} and \ref{hyp:FL_2} such that $\ell_i(y) = (f \infconv g)(y)$, where $\infconv$ denotes the infimal convolution, \textit{i.e.} $(f \infconv g)(y) = \inf_x f(x) + g(y-x)$.
Standard arguments about FL transforms state that $(f\infconv g)^\star = f^\star + g^\star$, so that if both $f$ and $g$ satisfy Assumptions \ref{hyp:FL_1} and \ref{hyp:FL_2}, so does $f \infconv g$.
This last example allows to deal with $\epsilon$-insensitive losses for instance (convolution of a loss and $\chi_{\mathcal{B}_\epsilon}$), the Huber loss (convolution of $\|.\|$ and $\|.\|^2$), or more generally all Moreau envelopes (convolution of a loss and $\frac{1}{2}\|.\|^2$).
\end{itemize}
\vspace{-0.33cm}
\qed


\subsection{Proof of \Cref{thm:double_rt_approx}}
\label{apx:double_rt_approx}
The proof of \Cref{thm:double_rt_approx} is straightforward: since the dual space $\widetilde{\mathcal{Y}}_m$ is of finite dimension $m$, the dual variable can be written as a linear combination of the $\{\psi_j\}_{j=1}^m$ to get \Cref{pbm:omega_pbm_approx}.

\subsection{Proof of \Cref{thm:loss_instantiation}}
\label{apx:loss_instantiation}


\subsubsection{$\epsilon$-Ridge -- from Problem $(P1)$ to $(D1)$}
\label{apx:other_pbms}

Applying \Cref{thm:dual_pbm} together with the Fenchel-Legendre transforms detailed in the proof of \Cref{prop:good_losses}, a dual to the $\epsilon$-Ridge regression primal problem is:
\begin{align*}
\min_{(\alpha_i)_{i=1}^n\in \mathcal{Y}^n} \quad &\frac{1}{2}\sum_{i=1}^n\|\alpha_i\|^2_\mathcal{Y} -\sum_{i=1}^n \left\langle \alpha_i, y_i\right\rangle_\mathcal{Y} + \epsilon \sum_{i=1}^n \|\alpha_i\|_\mathcal{Y} + \frac{1}{2\Lambda n} \sum_{i,j=1}^n \left\langle \alpha_i, \mathcal{K}(x_i, x_j) \alpha_j\right\rangle_\mathcal{Y},\\
\min_{(\alpha_i)_{i=1}^n\in \mathcal{Y}^n} \quad &\frac{1}{2} \sum_{i,j=1}^n \left\langle \alpha_i, \left(\delta_{ij} \mathbf{I}_\mathcal{Y} +  \frac{1}{\Lambda n}\mathcal{K}(x_i, x_j)\right) \alpha_j\right\rangle_\mathcal{Y} -\sum_{i=1}^n \left\langle \alpha_i, y_i\right\rangle_\mathcal{Y} + \epsilon \sum_{i=1}^n \|\alpha_i\|_\mathcal{Y}.
\end{align*}

By virtue of \Cref{thm:double_rt}, we known that the optimal $(\alpha_i)_{i=1}^n\in \mathcal{Y}^n$ are in $\bm{\mathsf{Y}}^n$.
After the reparametrization $\alpha_i = \sum_j \omega_{ij}~y_j$, the problem rewrites:
\begin{problem}\label{pbm:foo}
\min_{\Omega \in \mathbb{R}^{n \times n}} \quad \frac{1}{2} \mathbf{Tr}\left(\tilde{K}\Omega K^Y\Omega^\top\right) - \mathbf{Tr}\left(K^Y \Omega\right) + \epsilon \sum_{i=1}^n \sqrt{\left[\Omega K^Y \Omega^\top\right]_{ii}},
\end{problem}
with $\Omega$, $\tilde{K}$, $K^Y$ the $n \times n$ matrices such that $[\Omega]_{ij} = \omega_{ij}$, $\tilde{K} = \frac{1}{\Lambda n} K^X + \mathbf{I}_n$, and $[K^Y]_{ij} = \left\langle y_i, y_j\right\rangle_\mathcal{Y}$.

Now, let $K^Y = U \Sigma U^\top = \left(U\Sigma^{1/2}\right)\left(U\Sigma^{1/2}\right)^\top = VV^\top$ be the SVD of $K^Y$, and let $W = \Omega V$.
Notice that $K^Y$ is positive semi-definite, and can be made positive definite if necessary, so that $V$ is full rank, and optimizing with respect to $W$ is strictly equivalent to minimizing with respect to $\Omega$.
With this change of variable, \Cref{pbm:foo} rewrites:
\begin{problem}\label{pbm:bar}
\min_{W \in \mathbb{R}^{n \times n}} \quad \frac{1}{2} \mathbf{Tr}\left(\tilde{K}WW^\top\right) - \mathbf{Tr}\left(V^\top W\right) + \epsilon \|W\|_{2, 1},
\end{problem}
with $\|W\|_{2, 1} = \sum_i \|W_{i:}\|_2$ the row-wise $\ell_{2, 1}$ mixed norm of matrix $W$.
With $\tilde{K} = A^\top A$ the SVD of $\tilde{K}$, and $B$ such that $A^\top B = V$, one can add the constant term $\frac{1}{2}\mathbf{Tr}({A^\top}^{-1}VV^\top A^{-1}) = \frac{1}{2}\mathbf{Tr}(BB^\top)$ to the objective without changing \Cref{pbm:bar}.
One finally gets the Multi-Task Lasso problem:
\begin{equation*}
\min_{W \in \mathbb{R}^{n \times n}} \quad \frac{1}{2} \|AW - B\|_\mathrm{Fro}^2 + \epsilon \|W\|_{2, 1}.
\end{equation*}
\qed

We also emphasize that we recover the solution to the standard Ridge regression when $\epsilon = 0$.
Indeed, coming back to \Cref{pbm:foo} and differentiating with respect to $\Omega$, one gets:
\begin{equation*}
\tilde{K}\hat{\Omega}K^Y - K^Y = 0 \iff \hat{\Omega} = \tilde{K}^{-1},
\end{equation*}
which is exactly the standard kernel Ridge regression solution, see \textit{e.g.} \citetNew{brouard2016input2}.

Furthermore, notice that when $\mathcal{K}$ is not identity decomposable, but only satisfies \Cref{hyp:sum_decompo}, then \Cref{pbm:bar} cannot be factorized that easily.
Nonetheless, it admits a simple resolution, as detailed in the following lines.
After the $\Omega$ reparametrization, the problem writes
\begin{gather*}
\min_{\Omega \in \mathbb{R}^{n \times n}} \quad \frac{1}{2} \mathbf{Tr}(\Omega K^Y \Omega^\top) - \mathbf{Tr}\left(K^Y \Omega\right) + \epsilon \sum_{i=1}^n \sqrt{\left[\Omega K^Y \Omega^\top\right]_{i, i}} + \frac{1}{2 \Lambda n} \sum_{t=1}^T \mathbf{Tr}(K^X_t \Omega K^Y_t \Omega^\top),\\
\min_{W \in \mathbb{R}^{n \times n}} \quad \frac{1}{2} \mathbf{Tr}(WW^\top) + \frac{1}{2 \Lambda n} \sum_{t=1}^T \mathbf{Tr}(K^X_t W \tilde{K}^Y_t W^\top) - \mathbf{Tr}\left(V^\top W\right) + \epsilon \|W\|_{2, 1},
\end{gather*}
with $K^Y = VV^\top$, $W = \Omega V$, $\tilde{K}_t^Y = V^{-1} K_t^Y (V^\top)^{-1}$.
Due to the different quadratic terms, this problem cannot be summed up as a Multi-Task Lasso like before.
However, it may still be solved, \textit{e.g.} by proximal gradient descent.
Indeed, the gradient of the smooth term (\textit{i.e.} all but the $\ell_{2, 1}$ mixed norm) reads
\begin{equation}\label{eq:gradient}
W + \frac{1}{\Lambda n} \sum_{t=1}^T K^X_t W \tilde{K}^Y_t - V,
\end{equation}
while the proximal operator of the $\ell_{2, 1}$ mixed norm is
\begin{equation*}
\text{prox}_{\epsilon~\|~\cdot~\|_{2, 1}}(W) = \left(\begin{matrix}|\\\text{prox}_{\epsilon~\|~\cdot~\|_2}(W_{i:})\\|\end{matrix}\right) = \left(\begin{matrix}|\\\left(1 - \frac{\epsilon}{\|W_{i:}\|_2}\right)_+W_{i:}\\|\end{matrix}\right) = \left(\begin{matrix}|\\\text{BST}(W_{i:}, \epsilon)\\|\end{matrix}\right).
\end{equation*}
Hence, even in the more involved case of an OVK satisfying only \Cref{hyp:sum_decompo}, we have designed an efficient algorithm to compute the solutions to the dual problem.


\subsubsection{$\kappa$-Huber -- from Problem $(P2)$ to $(D2)$}
\label{apx:huber}

Basic manipulations give the Fenchel-Legendre transforms of the Huber loss:
\begin{align*}
\Big(y \mapsto \ell_{H, \kappa}(y - y_i)\Big)^\star(\alpha) &= \left(\kappa\|\cdot\|_\mathcal{Y} \infconv \frac{1}{2}\|\cdot\|_\mathcal{Y}^2\right)^\star(\alpha) + \left\langle \alpha, y_i\right\rangle_\mathcal{Y},\\
&= \left(\kappa\|\cdot\|_\mathcal{Y} \right)^\star(\alpha) + \left(\frac{1}{2}\|\cdot\|_\mathcal{Y}^2\right)^\star(\alpha) + \left\langle \alpha, y_i\right\rangle_\mathcal{Y},\\
&= \chi_{\mathcal{B}_\kappa}(\alpha) + \frac{1}{2}\|\alpha\|_\mathcal{Y}^2 + \left\langle \alpha, y_i\right\rangle_\mathcal{Y}.
\end{align*}

Following the same lines as for as for the $\epsilon$-Ridge regression, the dual problem writes
\begin{equation*}
\min_{(\alpha_i)_{i=1}^n\in \mathcal{Y}^n} \quad \frac{1}{2} \sum_{i,j=1}^n \left\langle \alpha_i, \left(\delta_{ij} \mathbf{I}_\mathcal{Y} +  \frac{1}{\Lambda n}\mathcal{K}(x_i, x_j)\right) \alpha_j\right\rangle_\mathcal{Y} -\sum_{i=1}^n \left\langle \alpha_i, y_i\right\rangle_\mathcal{Y} + \sum_{i=1}^n \chi_\kappa(\|\alpha_i\|_\mathcal{Y}),
\end{equation*}
or again after the reparametrization in $\Omega$
\begin{equation*}
\begin{aligned}
&\min_{\Omega \in \mathbb{R}^{n \times n}} \quad &&\frac{1}{2} \mathbf{Tr}\left(\tilde{K}\Omega K^Y\Omega^\top\right) - \mathbf{Tr}\left(K^Y \Omega\right)\\
&~~~\text{s.t.} && \sqrt{\left[\Omega K^Y \Omega^\top\right]_{ii}} \le \kappa \qquad \forall i \le n
\end{aligned}
\end{equation*}
The same change of variable permits to conclude.\qed

When $\mathcal{K}$ is not identity decomposable, but only satisfies \Cref{hyp:sum_decompo}, the problem rewrites
\begin{equation*}
\begin{aligned}
&\min_{W \in \mathbb{R}^{n \times n}} ~~ &&\frac{1}{2} \mathbf{Tr}(WW^\top) + \frac{1}{2 \Lambda n} \sum_{t=1}^T \mathbf{Tr}(K^X_t W \tilde{K}^Y_t W^\top) - \mathbf{Tr}\left(V^\top W\right),\\
&~~~\text{s.t.} && \|W_{i:}\|_2 \le \kappa \quad \forall i \le n,
\end{aligned}
\end{equation*}
Again, the gradient term is given by \Cref{eq:gradient}, while the projection is similar to the identity decomposable case.
The only change thus occurs in the gradient step of \Cref{alg:pgd}, with a replacement by the above formula.

Notice that if $\kappa$ tends to infinity, the problem is unconstrained, and one also recovers the standard Ridge regression solution.

\subsubsection{$\epsilon$-SVR -- from Problem $(P3)$ to $(D3)$}

The proof is similar to the above derivations except that the term $\sum_i\|\alpha_i\|_\mathcal{Y}^2$ does not appear in the dual, hence the change of matrix $\tilde{K}$.
Instead, the dual problem features both the $\ell_{2, 1}$ penalization and the $\ell_{2, \infty}$ constraint.
\qed

\subsection{Proof of \Cref{thm:huber_l2}}
The proof is similar to \Cref{apx:huber}, with the finite representation coming from \Cref{thm:double_rt_approx}.


\subsection{Proof of \Cref{thm:stability}}
\label{sec:apx_stability}

In this section, we detail the derivation of constants in \Cref{tab:constants}.


\subsubsection{$\epsilon$-SVR}

Using that the null function is part of the vv-RKHS, it holds
\begin{equation*}
\frac{\Lambda}{2} \|\hsample\|_{\mathcal{H}_\mathcal{K}}^2 \le \hat{\mathcal{R}}_n(\hsample) \le \hat{\mathcal{R}}_n(0_{\mathcal{H}_\mathcal{K}}) \le M_\mathcal{Y} - \epsilon, \qquad \text{or again} \qquad \|\hsample\|_{\mathcal{H}_\mathcal{K}} \le \sqrt{\frac{2}{\Lambda}(M_\mathcal{Y} - \epsilon)}.
\end{equation*}

Furthermore, the reproducing property and \Cref{hyp:bounded_kernel} give that for any $x\in \mathcal{X}$ and any $h \in \mathcal{H}_\mathcal{K}$ it holds
\begin{equation*}
\|h(x)\|^2 = \left\langle \mathcal{K}(\cdot, x)\mathcal{K}(\cdot, x)^\# h, h\right\rangle_{\mathcal{H}_\mathcal{K}} \le \left\|\mathcal{K}(\cdot, x)\mathcal{K}(\cdot, x)^\#\right\|_\text{op} \|h\|^2_{\mathcal{H}_\mathcal{K}} \le \left\|\mathcal{K}(x, x) \right\|_\text{op} \|h\|^2_{\mathcal{H}_\mathcal{K}} \le \gamma^2 \|h\|^2_{\mathcal{H}_\mathcal{K}}.
\end{equation*}

Therefore, one gets that for any realization $(x, y) \in \mathcal{X} \times \mathcal{Y}$ of $(X, Y)$ it holds
\begin{equation*}
\ell(\hsample(x), y) = \max(\|y - \hsample(x)\|_\mathcal{Y}-\epsilon, 0) \le M_y - \epsilon + \|\hsample(x)\|_\mathcal{Y} \le \sqrt{M_\mathcal{Y} - \epsilon}\left(\gamma\sqrt{\frac{2}{\Lambda}} + \sqrt{M_\mathcal{Y} - \epsilon}\right).
\end{equation*}

This gives $M$. As for $C$, one has
\begin{equation*}
\ell(\hsample(x), y) - \ell(\hsamplei(x), y) = \max(\|y - \hsample(x)\|_\mathcal{Y} - \epsilon, 0) - \max(\|y - \hsamplei(x)\|_\mathcal{Y} - \epsilon, 0).
\end{equation*}
If both norms are smaller than $\epsilon$, then any value of $C$ fits.
If both norms are greater than $\epsilon$, the difference reads
\begin{equation*}
\|y - \hsample(x)\|_\mathcal{Y} - \|y - \hsamplei(x)\|_\mathcal{Y} \le \|\hsample(x) - \hsamplei(x)\|_\mathcal{Y}.
\end{equation*}
If only one norm is greater than $\epsilon$ (we write it only for $\hsample$ as it is symmetrical), the difference may be rewritten
\begin{equation*}
\|y - \hsample(x)\|_\mathcal{Y} - \epsilon \le \|y - \hsample(x)\|_\mathcal{Y} - \|y - \hsamplei(x)\|_\mathcal{Y} \le \|\hsample(x) - \hsamplei(x)\|_\mathcal{Y}.
\end{equation*}
Hence we get $C=1$.


\subsubsection{$\epsilon$-Ridge}

Using the same reasoning as for the $\epsilon$-SVR, one has
\begin{equation}\label{eq:norm_funcs}
\|\hsample\|_{\mathcal{H}_\mathcal{K}} \le \sqrt{\frac{2}{\Lambda}}(M_\mathcal{Y} - \epsilon) \qquad \text{and} \qquad \|\hsamplei\|_{\mathcal{H}_\mathcal{K}} \le \sqrt{\frac{2}{\Lambda}}(M_\mathcal{Y} - \epsilon).
\end{equation}

Therefore, for any realization $(x, y) \in \mathcal{X} \times \mathcal{Y}$ of $(X, Y)$ it holds
\begin{equation*}
\ell(\hsample(x), y) = \max(\|y - \hsample(x)\| - \epsilon, 0)^2 \le (\|y\|_\mathcal{Y} - \epsilon + \|\hsample(x)\|_\mathcal{Y})^2 \le (M_\mathcal{Y} - \epsilon)^2\left(1 + \frac{2\sqrt{2}\gamma}{\sqrt{\Lambda}} + \frac{2\gamma^2}{\Lambda}\right).
\end{equation*}

As for $C$, one has
\begin{equation*}
\ell(\hsample(x), y) - \ell(\hsamplei(x), y) = \max(\|y - \hsample(x)\|_\mathcal{Y} - \epsilon, 0)^2 - \max(\|y - \hsamplei(x)\|_\mathcal{Y} - \epsilon, 0)^2.
\end{equation*}
If both norms are smaller than $\epsilon$, any $C$ fits.
If both are larger than $\epsilon$, using \Cref{eq:norm_funcs} the difference becomes
\begin{align*}
&\left(\|y - \hsample(x)\|_\mathcal{Y} + \|y - \hsamplei(x)\|_\mathcal{Y} - 2\epsilon\right)\left(\|y - \hsample(x)\|_\mathcal{Y} - \|y - \hsamplei(x)\|_\mathcal{Y}\right),\\[0.15cm]
\le~& 2(M_\mathcal{Y} - \epsilon)\left(1 + \frac{\gamma\sqrt{2}}{\sqrt{\Lambda}}\right)\|\hsample(x) - \hsamplei(x)\|_\mathcal{Y}.
\end{align*}
If only one norm is greater than $\epsilon$ (again, the analysis is symmetrical), the difference may be rewritten
\begin{align*}
\left(\|y - \hsample(x)\|_\mathcal{Y} - \epsilon\right)^2 &\le \left(\|y - \hsample(x)\|_\mathcal{Y} - \|y - \hsamplei(x)\|_\mathcal{Y}\right)^2 \le \|\hsample(x) - \hsamplei(x)\|^2_\mathcal{Y},\\[0.15cm]
&\le \left(\|\hsample(x)\|_\mathcal{Y} + \|\hsamplei(x)\|_\mathcal{Y}\right)\|\hsample(x) - \hsamplei(x)\|_\mathcal{Y},\\[0.05cm]
&\le 2(M_\mathcal{Y} - \epsilon)\frac{\gamma\sqrt{2}}{\sqrt{\Lambda}}\|\hsample(x) - \hsamplei(x)\|_\mathcal{Y}.
\end{align*}
In every case $C = 2(M_\mathcal{Y} - \epsilon)\left(1 + \gamma\sqrt{2}/\sqrt{\Lambda}\right)$ works, hence the conclusion.


\subsubsection{$\kappa$-Huber}

Using the same techniques, one gets
\begin{equation*}
\|\hsample\|_{\mathcal{H}_\mathcal{K}} \le \sqrt{\frac{2\kappa}{\Lambda}\left(M_\mathcal{Y} - \frac{\kappa}{2}\right)} \qquad \text{and} \qquad \|\hsamplei\|_{\mathcal{H}_\mathcal{K}} \le \sqrt{\frac{2\kappa}{\Lambda}\left(M_\mathcal{Y} - \frac{\kappa}{2}\right)},
\end{equation*}
and for any realization $(x, y) \in \mathcal{X} \times \mathcal{Y}$ of $(X, Y)$
\begin{equation*}
\ell(\hsample(x), y) \le \kappa\sqrt{M_\mathcal{Y} - \frac{\kappa}{2}}\left(\frac{\gamma\sqrt{2\kappa}}{\sqrt{\Lambda}} + \sqrt{M_\mathcal{Y} - \frac{\kappa}{2}}\right).
\end{equation*}

If both norms are greater than $\kappa$, the difference $\ell(\hsample(x), y) - \ell(\hsamplei(x), y)$ writes
\begin{equation*}
\kappa\left(\|\hsample(x) - y\|_\mathcal{Y} - \frac{\kappa}{2}\right) - \kappa\left(\|\hsamplei(x) - y\|_\mathcal{Y} - \frac{\kappa}{2}\right) \le \kappa \|\hsample(x) - \hsamplei(x)\|_\mathcal{Y}.
\end{equation*}
If only one norm is greater than $\kappa$, one may upperbound the difference using the previous writing
\begin{equation*}
\kappa\left(\|\hsample(x) - y\|_\mathcal{Y} - \frac{\kappa}{2}\right) - \frac{1}{2} \|\hsamplei(x) - y\|_\mathcal{Y}^2 \le \kappa\left(\|\hsample(x) - y\|_\mathcal{Y} - \frac{\kappa}{2}\right) - \kappa\left(\|\hsamplei(x) - y\|_\mathcal{Y} - \frac{\kappa}{2}\right).
\end{equation*}
If both are smaller than $\kappa$, the difference becomes
\begin{align*}
\frac{1}{2} & \|\hsample(x) - y\|_\mathcal{Y}^2 - \frac{1}{2} \|\hsamplei(x) - y\|_\mathcal{Y}^2,\\
&= \frac{1}{2}\left(\|\hsample(x) - y\|_\mathcal{Y} + \|\hsamplei(x) - y\|_\mathcal{Y}\right)\left(\|\hsample(x) - y\|_\mathcal{Y} - \|\hsamplei(x) - y\|_\mathcal{Y}\right),\\
&\le \kappa \|\hsample(x) - \hsamplei(x)\|_\mathcal{Y},
\end{align*}
so that $C = \kappa$.


\subsection{Further Admissible Kernels for \Cref{hyp:stable_ovk}}
\label{apx:suit_ker}

In the continuation of \Cref{rmk:stable_ovk}, we now exhibit several types of OVK that satisfy \Cref{hyp:stable_ovk}.
\begin{proposition}
The following Operator-Valued Kernels satisfy \Cref{hyp:stable_ovk}:\vspace{-0.2cm}
\begin{enumerate}[label=(\roman*)]
\item $\forall s, t \in \mathcal{X}^2, ~~ \mathcal{K}(s, t) = \sum_i~k_i(s, t)~y_i \otimes y_i$, \hfill with $k_i$ positive semi-definite (p.s.d.) scalar kernels for all $i \le n$.
\item $\forall s, t \in \mathcal{X}^2, ~~ \mathcal{K}(s, t) = \sum_i~\mu_i~k(s, t)~y_i \otimes y_i$, \hfill with $k$ a p.s.d. scalar kernel and $\mu_i \ge 0$ for all $i \le n$.
\item $\forall s, t \in \mathcal{X}^2, ~~ \mathcal{K}(s, t) = \sum_i~k(s, x_i) k(t, x_i)~y_i \otimes y_i$,
\item $\forall s, t \in \mathcal{X}^2, ~~ \mathcal{K}(s, t) = \sum_{i,j}~k_{ij}(s, t)~(y_i + y_j)\otimes(y_i + y_j)$, \hfill with $k_{ij}$ p.s.d. scalar kernels for all $i,j \le n$.
\item $\forall s, t \in \mathcal{X}^2, ~~ \mathcal{K}(s, t) = \sum_{i,j}~\mu_{ij}~k(s,t)~(y_i + y_j)\otimes(y_i + y_j)$, \hfill with $k$ a p.s.d. scalar kernel and $\mu_{ij} \ge 0$.
\item $\forall s, t \in \mathcal{X}^2, ~~ \mathcal{K}(s, t) = \sum_{i,j}~k(s, x_i, x_j)k(t, x_i, x_j)~(y_i + y_j)\otimes(y_i + y_j)$.
\end{enumerate}
\end{proposition}

\begin{proof}
~\\[-0.6cm]
\begin{enumerate}[label=$(\roman*)$]
\item For all $(s_k, z_k)_{k \le n} \in (\mathcal{X} \times \mathcal{Y})^n$, it holds: $\sum_{k,l}\left\langle z_k, \mathcal{K}(s_k, s_l)z_k\right\rangle_\mathcal{Y} = \sum_i \sum_{k, l} k_i(s, t) \left\langle z_k, y_i\right\rangle_\mathcal{Y}\left\langle z_l, y_i\right\rangle_\mathcal{Y}$, which is\\[0.15cm]positive by the positiveness of the scalar kernels $k_i$'s.
Notice that $(ii)$ and $(iii)$ are then particular cases of $(i)$.
\item is an application of $(i)$, as a kernel remains p.s.d. through positive multiplication.
Observe that this kernel is separable.
\item is also a direct application of $(i)$, kernel $k': s, t \mapsto k(s, x_i)k(t, x_i)$ being indeed p.s.d. for all function $k$ and point $x_i$.
\item is proved similarly to $(i)$. The arguments used for $(ii)$ and $(iii)$ also makes $(v)$ and $(vi)$ direct applications of $(iv)$.
\end{enumerate}
Finally, notice that for $(iv)$, $(v)$ and $(vi)$, any linear combination $(\nu_i y_i + \nu_j y_j) \otimes (\nu_i y_i + \nu_j y_j)$, with $0 \le \nu_i \le 1$ for all $i \le n$, could have been used instead of $(y_i + y_j)\otimes(y_i + y_j)$.
\end{proof}


\section{Loss Functions Illustrations}
\label{apx:loss_fig}

In this section, we provide illustrations of the loss functions we used to promote sparsity and robustness.
This includes $\epsilon$-insensitive losses (\Cref{def:eps_insensitive,def:eps_svr_ridge}, \Cref{fig:eps_svr,fig:eps_ridge}) and the $\kappa$-Huber loss (\Cref{def:Huber}, \Cref{fig:huber}).
First introduced for real outputs, their formulations as infimal convolutions allows for a generalization to any Hilbert space, either of finite dimension (as in \citetNew{sangnier2017data2}) or not, which is the general case addressed in the present paper.
The $\epsilon$-insensitive loss functions promote sparsity, as reflected in the corresponding dual problems (see \Cref{thm:loss_instantiation}, Problems $(D1)$ and $(D3)$ therein) and the empirical results (\Cref{fig:yeast_ekrr,fig:yeast_esvr}).
On the other hand, losses whose slopes asymptotically behave as $||\cdot ||_{\mathcal{Y}}$ instead of $||\cdot||_{\mathcal{Y}}^2$ (such as the $\kappa$-Huber or the $\epsilon$-SVR loss) encourage robustness through a resistance to outliers.
Indeed, under such a setting, residuals of high norm contribute less to the gradient and have a minor influence on the model output.

\begin{figure*}[!h]
\center
\includegraphics[height=4.85cm]{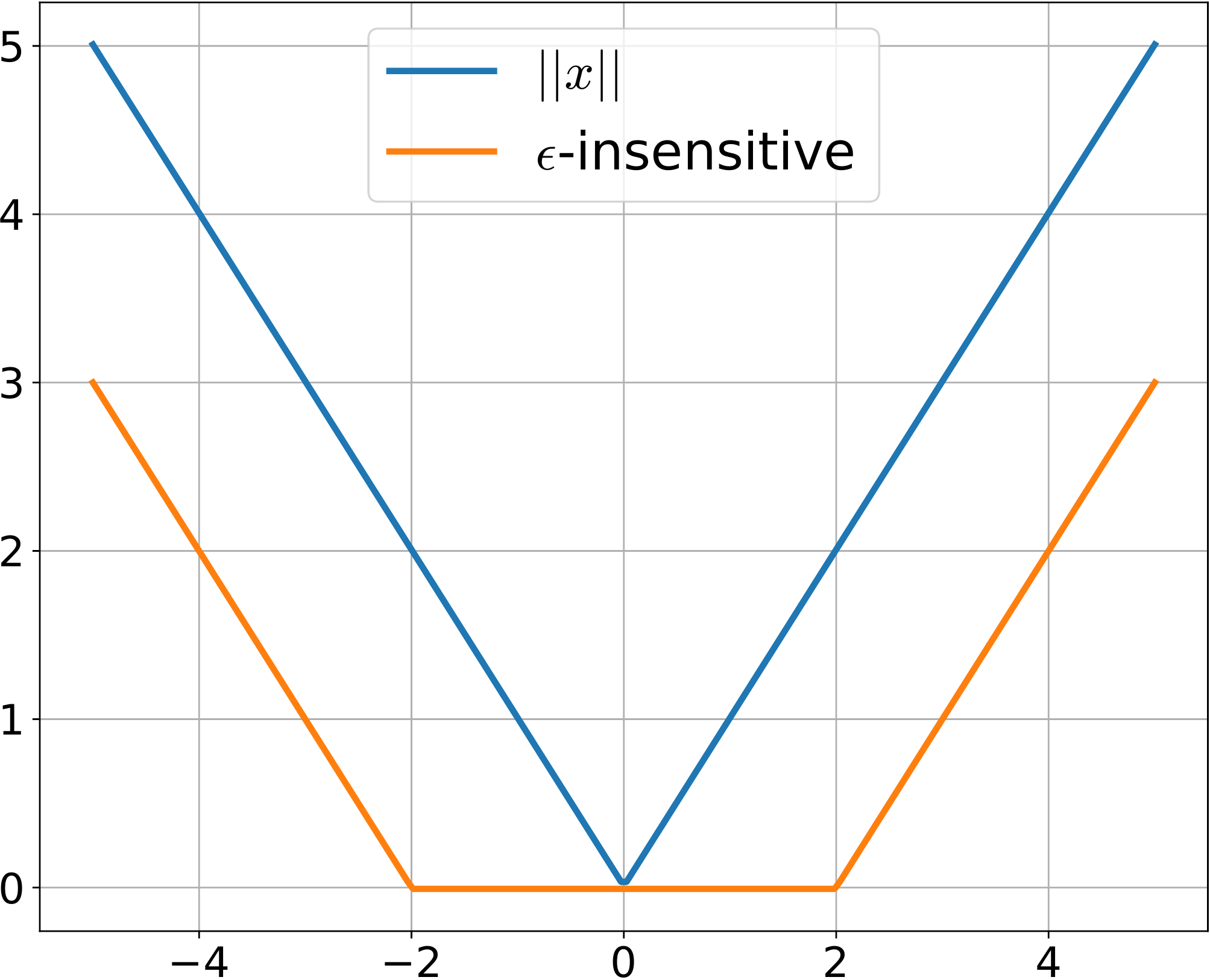}\hspace{1.7cm}
\includegraphics[height=5.1cm]{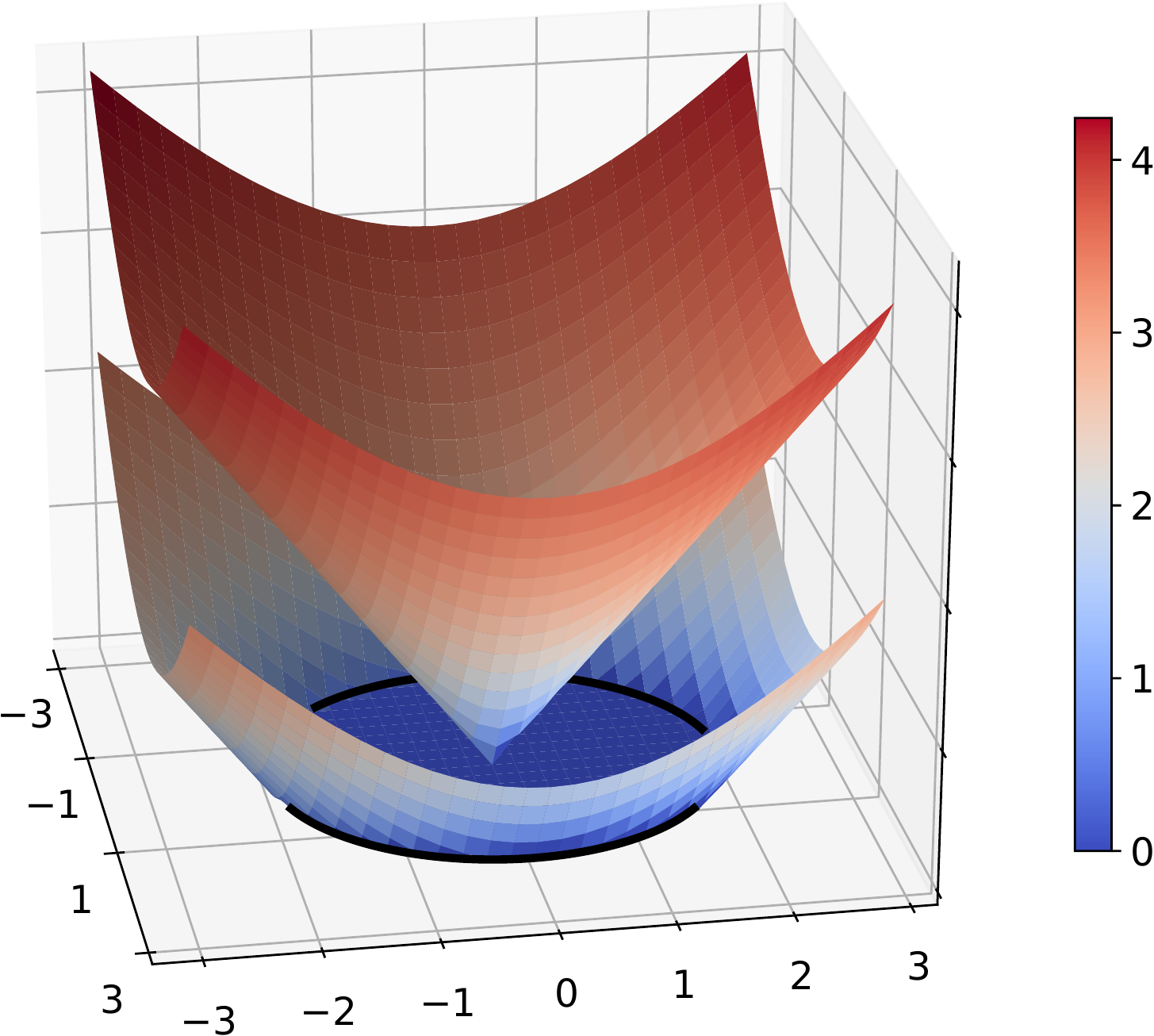}
\caption{Standard and $\epsilon$-insensitive versions of the SVR loss in $1$ and $2$ dimensions ($\epsilon=2$).}
\label{fig:eps_svr}
\end{figure*}

\begin{figure*}[!h]
\center
\includegraphics[height=4.85cm]{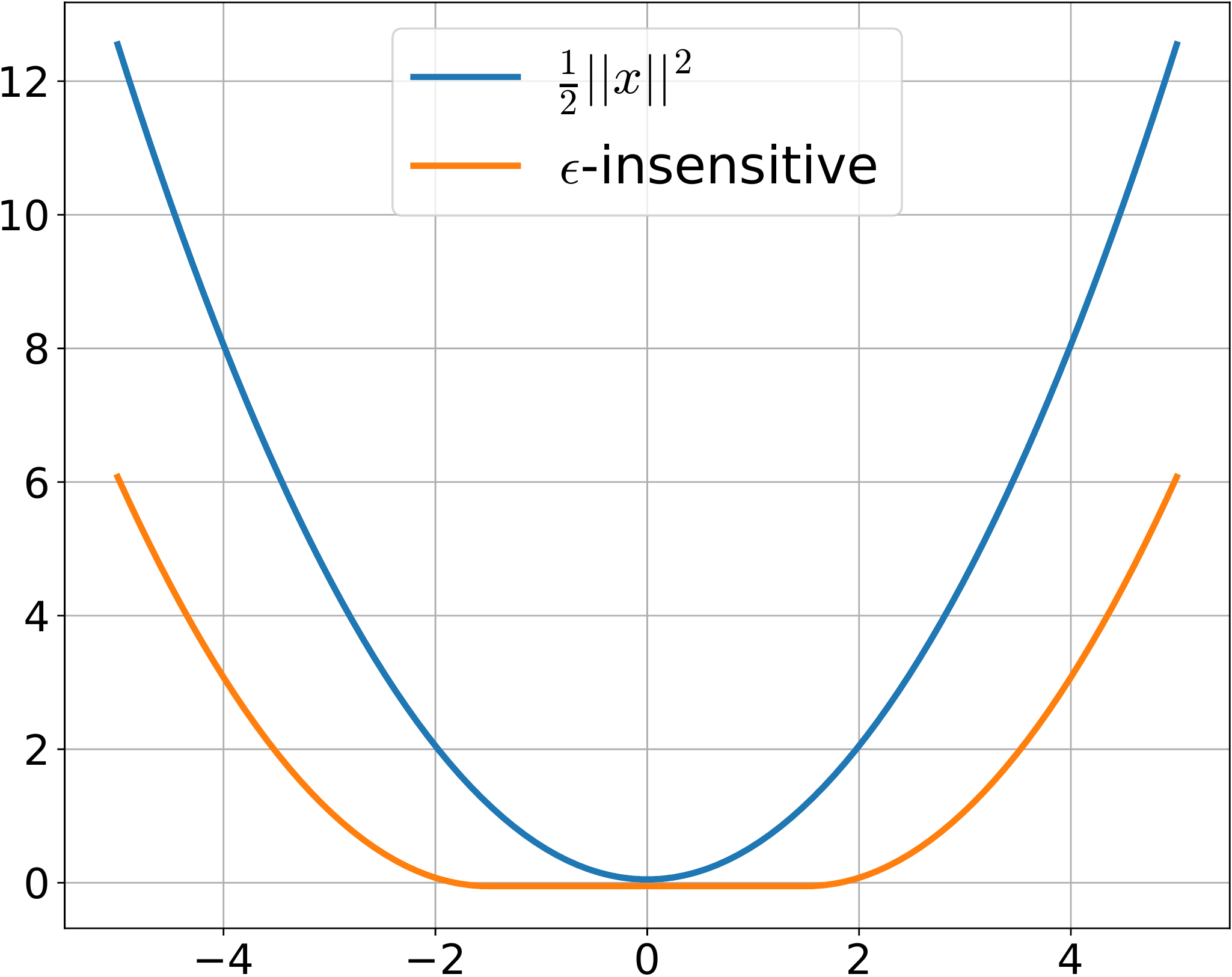}\hspace{1.7cm}
\includegraphics[height=5.1cm]{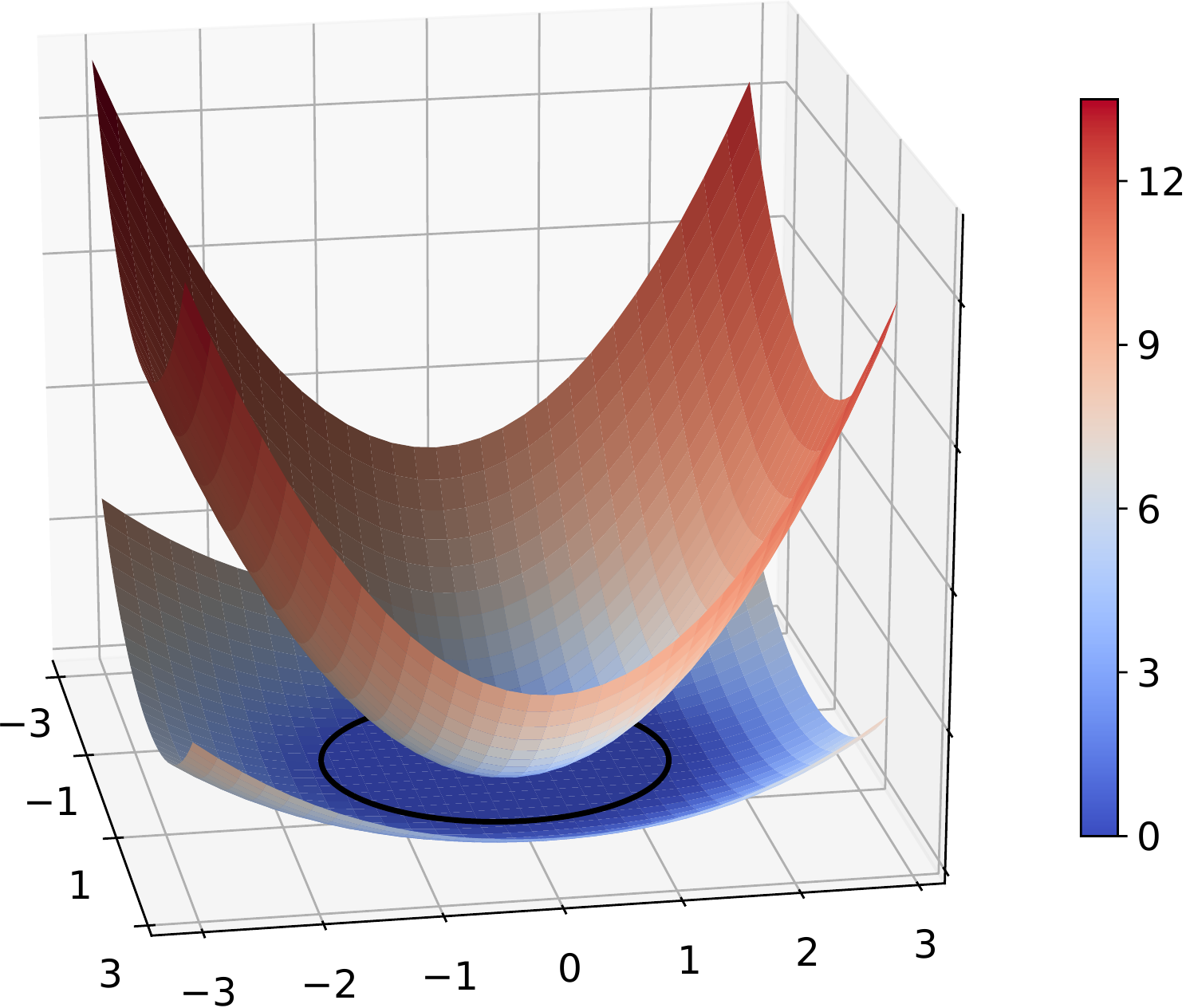}
\caption{Standard and $\epsilon$-insensitive versions of the square loss in $1$ and $2$ dimensions ($\epsilon=1.5$).}
\label{fig:eps_ridge}
\end{figure*}

\begin{figure*}[!t]
\center
\includegraphics[height=4.85cm]{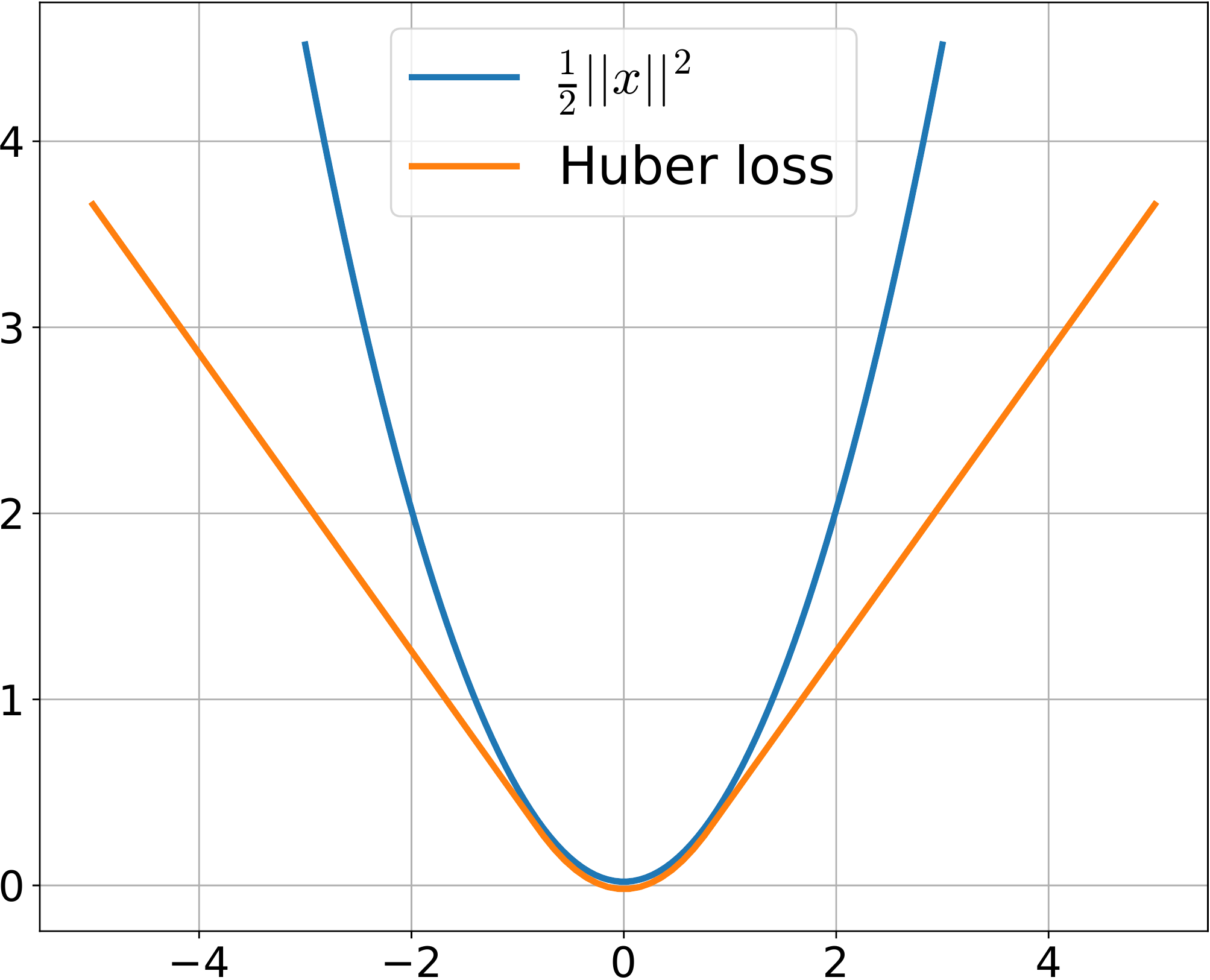}\hspace{1.7cm}
\includegraphics[height=5.1cm]{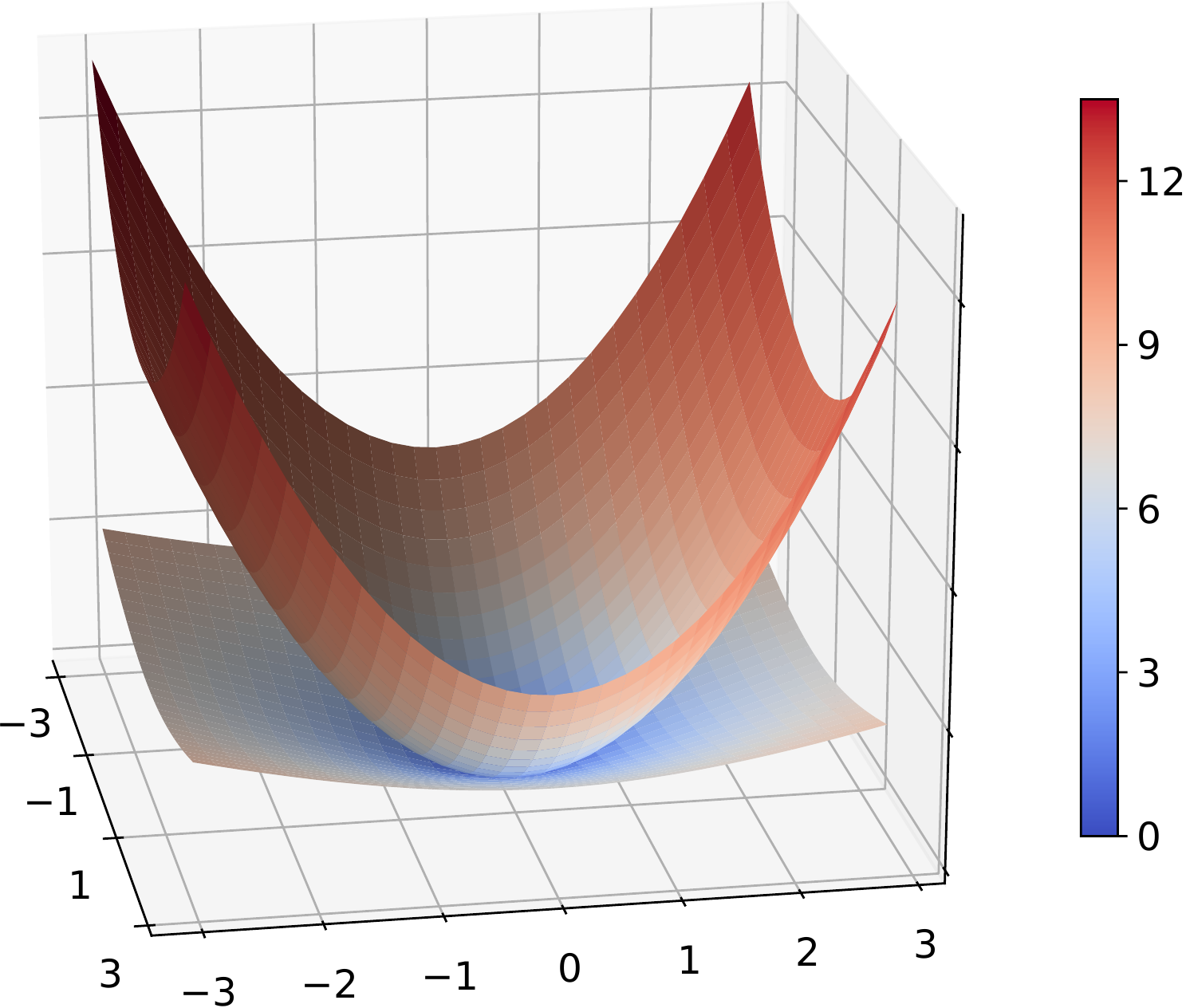}
\caption{Standard square loss and Huber loss in $1$ and $2$ dimensions ($\kappa=0.8$).}
\label{fig:huber}
\end{figure*}


\section{Numerical Experiments and Code}
\label{apx:expes}

\subsection{Provided Code}

The Python code used to generate the plots and tables of the article is provided.
The \texttt{README} file in the code folder contains instructions for quickly reproducing (part of) the plots.
All implemented methods may be run on other datasets/problems.

\subsection{Detailed Protocols}

\subsubsection{Structured Prediction}

{\bf YEAST Dataset Description.}
YEAST\footnote{\href{https://www.csie.ntu.edu.tw/~cjlin/libsvmtools/datasets/}{https://www.csie.ntu.edu.tw/~cjlin/libsvmtools/datasets/}} is a publicly available multi-label classification dataset used as a benchmark in several structured prediction articles.
We compared our approach, with the same train/test decomposition, to those presented in \citetNew{elisseeff2002kernel2}, \citetNew{finley2008training2} and \citetNew{belanger2016structured2}.
The size of the training set is $1500$, the test set is of size $917$. The problem consists in predicting the functional classes of a gene.
The inputs are micro-array expression data (representing the genes) of dimension $p = 103$.
The outputs are multi-label vectors of size \mbox{$d = 14$} representing the possible functional classes of the genes.
The average number of labels is $4.2$.
These $14$ functional classes correspond to the first level of a tree that structures a much bigger set of possible functional classes.

{\bf Experimental protocol: Comparison with other methods.}
In \Cref{tab:yeast}, we reported the Hamming error on the test set obtained by each method.
The results obtained by SSVM and SPENS are extracted from \citetNew{finley2008training2} and \citetNew{belanger2016structured2}.
For our approach and its three variants ($\epsilon$-KRR, $\kappa$-Huber, $\epsilon$-SVR), each hyper-parameter ($\Lambda$, $\epsilon$, or $\kappa$) has been selected by estimating the Mean Squared Error (MSE) through a $5$-fold cross-validation computed on the training set.
We used an input Gaussian kernel with a fixed bandwidth equal to $1$.

{\bf Experimental protocol: Cross-Effect of $\epsilon$ and $\Lambda$ on sparsity and MSE.}
In order to measure the effect of the different hyperparameters and study their interrelations, we have computed the $5$-fold cross-validation MSE and sparsity/saturation for several values of $\Lambda$ and $\epsilon$/$\kappa$.
The input kernel is still Gaussian with bandwidth $1$.
The results are plotted in \Cref{fig:mse_yeast,fig:sparse_yeast} for the $\epsilon$-KRR, and in \Cref{fig:yeast_esvr,fig:yeast_huber} for the $\epsilon$-SVR and $\kappa$-Huber.
In \Cref{fig:sparse_yeast}, we have measured sparsity through the number of training data which are discarded, \textit{i.e.} not used in the finite representation of the $\epsilon$-KRR model.
The $\kappa$-Huber saturation is assessed in a similar fashion: it corresponds to the number of training data whose associated coefficient saturates the norm constraint (see \Cref{thm:loss_instantiation}, Problem $(D2)$ therein).
Simplified versions of these graphs may be quickly reproduced using the code attached (see \texttt{README} file).

{\bf Metabolite identification dataset description.}
We next tested our method on a harder problem: that of metabolite \mbox{identification} \citepNew{brouard2016fast2}.
The goal is to predict a metabolite (small molecule) thanks to its mass spectrum.
The difficulty comes from the reduced size of the training set ($n = 6974$) compared to the high dimension of the outputs ($d = 7593$).
Input Output Kernel Regression (IOKR, see \citetNew{brouard2016fast2,brouard2016input2}) with a Tanimoto-Gaussian kernel is state of the art on this problem.

{\bf Experimental protocol.}
We investigate the advantages of substituting the Ridge Regression for the $\epsilon$-KRR, $\kappa$-Huber, and $\epsilon$-SVR.
Outputs are embedded in an infinite dimensional space through the use of the Tanimoto-Gaussian kernel (with bandwidth $\gamma=0.72$).
We compare the different algorithms' performances on a set of $6974$ mass spectra through the top-$k$ accuracies for $k \in \{1, 10, 20\}$.
We give the average $5$-fold top-$k$ accuracies (\Cref{tab:metabolite}).
The $5$ folds have been chosen such that a metabolite does not appear in two different folds (zero-shot learning setting).

\subsubsection{Structured Representation Learning}

{\bf Dataset Description.}
Robust structured representation learning was tested on a drug dataset, introduced in \citetNew{su2010structured2}, and extracted from the NCI-Cancer database.
This dataset features a set of molecules that are represented through a Gram matrix of size $2303 \times 2303$ obtained with a Tanimoto kernel.
Tanimoto kernels (see \citetNew{ralaivola2005graph2} for details) are a common way to compare labeled graphs by means of a bag-of-sequences approach.

{\bf Experimental protocol: Robust KAE.}
We computed the mean $5$-fold cross-validation Mean Squared Error.
The first layer uses a linear kernel.
But since inputs (and outputs) are kernelized -- only the $2303 \times 2303$ Gram matrix is provided for learning --, the first layer may also be seen as a function from the associated Tanimoto-RKHS, applied to the molecules.
The second layer uses a Gaussian kernel.
The regularization parameters for the two layers have been fixed to $\Lambda = 1e-6$, and the inner dimension has been set to $p = 200$.
In \Cref{fig:eps_kae} is plotted the MSE and the sparsity (discarded training data) for several values of $\epsilon$ in order to assess the effect of the regularization.
We used an existing source code from \citetNew{laforgue2019autoencoding2}\footnote{\href{https://github.com/plaforgue/kae}{github.com/plaforgue/kae}}, that has been adapted to our needs.
The IOKR resolution part, materialized by the \texttt{compute\_N\_L} function therein, has been replaced by the \texttt{compute\_Omega} function of the \texttt{IOKR\_plus} class in the attached code.

\subsubsection{Robust Function-to-Function Regression}

{\bf Dataset Description.}
The task at hand consists in predicting lip acceleration from electromyography (EMG) signals of the corresponding muscle \citepNew{ramsay2007applied2}.
The dataset\footnote{\href{http://www.stats.ox.ac.uk/~silverma/fdacasebook/lipemg.html}{http://www.stats.ox.ac.uk/~silverma/fdacasebook/lipemg.html}} includes 32 samples of time series obtained by recording a subject saying ``say bob again'', that are noted $(x_i,y_i)_{i=1}^{32}$.
Each time series is of length $64$.
To assess the performance of our method in presence of outliers, we created $4$ outliers by picking randomly some $(x_i)_{i=1}^4$ and adding to the dataset the samples $(x_i, -1.2*y_i)_{i=1}^4$.

{\bf Experimental protocol.}
As the number of samples is small, one can use the Leave One Out (LOO) generalization error as a measure of the model performance.
We first used it with plain Ridge Regression \citepNew{kadri2016ovk2} to select the best hyperparameter $\Lambda$.
Then, we tested robustness by computing the LOO generalization error of a model output by solving \Cref{pbm:huber_dual_l2} for various $\kappa$ (see \Cref{fig:huber_fonctional}, that may also be reproduced from the attached code).
For the $\{\psi_j\}_{j=1}^m$ we used the sine and cosine basis of $L^2([0,1])$, i.e. $\forall l \leq \frac{m}{2}$ and $\theta \in [0,1]$, $\psi_{2l}(\theta)= \sqrt{2} \cos(2\pi l \theta)$ and $\psi_{2l+1}(\theta)= \sqrt{2} \sin(2\pi l \theta)$.
The number of basis function was set to $m=16$, so that we get the first $8$ cosines and sines of the basis.
The chosen associated eigenvalues are $\lambda_{2l} =\lambda_{2l+1} = \frac{1}{(1+j)^2}$.
We used as an input kernel the integral Laplacian $k_{\mathcal{X}}(x_1,x_2) = \int_{0}^1 \exp{(-7|x_1(\theta)-x_2(\theta)|)} \mathrm{d}\theta$.

\subsection{Additional Figures}

We now provide analogues to \Cref{fig:mse_yeast,fig:sparse_yeast} for the $\epsilon$-SVR and $\kappa$-Huber.
The $\epsilon$-Ridge graphs are first recalled.
Notice that simplified versions of these plots may be easily reproduced from the attached code.

The $\epsilon$-KRR (\Cref{fig:yeast_ekrr}) appears as a natural regularized version of the plain KRR.
For small values of $\Lambda$, the regularization effect of the $\epsilon$ induces a smaller MSE.
This phenomenon is achieved for a wide range of $\Lambda$ and $\epsilon$, and coincides with an important sparsity.
The counterpart is that no value of $\epsilon$ clearly allows to outperform the standard KRR for its optimal $\Lambda$.
The $\epsilon$-KRR may rather be used as an implicit regularization preventing from a cross-validation on $\Lambda$.

The $\epsilon$-SVR (\Cref{fig:yeast_esvr}) shares analogous characteristics for the small $\Lambda$ regime.
However, it further produces predictors with smaller MSE than the best KRR one.
This furthermore coincides with a peak in the sparsity.

The $\kappa$-Huber (\Cref{fig:yeast_huber}) has a quite different behavior.
When $\Lambda$ tends to $0$, the constraint (see Problem $(D2)$) is vacuous for all $\kappa$, and one asymptotically recovers the standard KRR.
The optimal $\Lambda$ now changes with $\kappa$, and better performances than the KRR for the best $\Lambda$ are regularly attained.

\begin{figure*}[!h]
\center
\includegraphics[height=6.2cm]{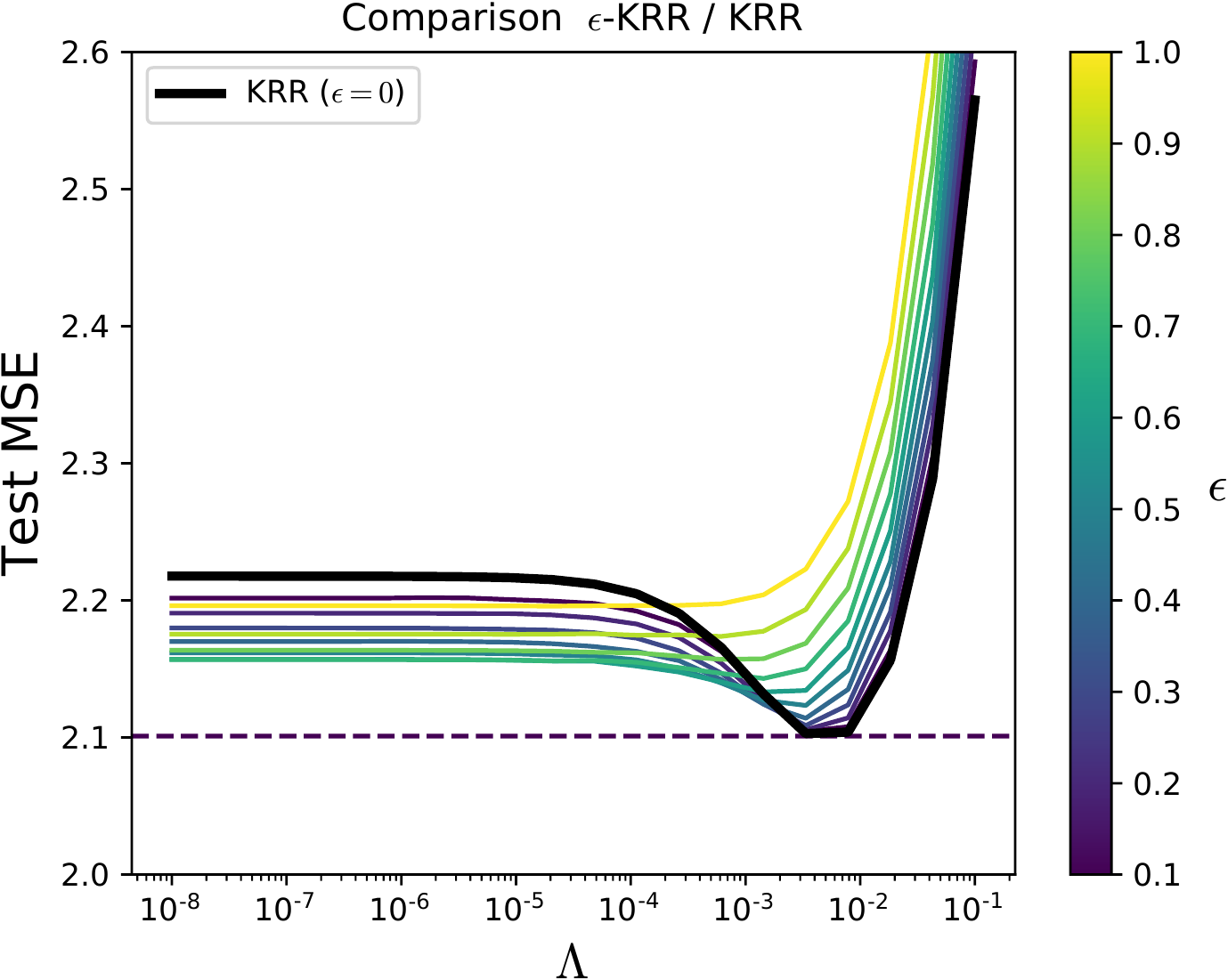}\hspace{0.5cm}
\includegraphics[height=6.2cm]{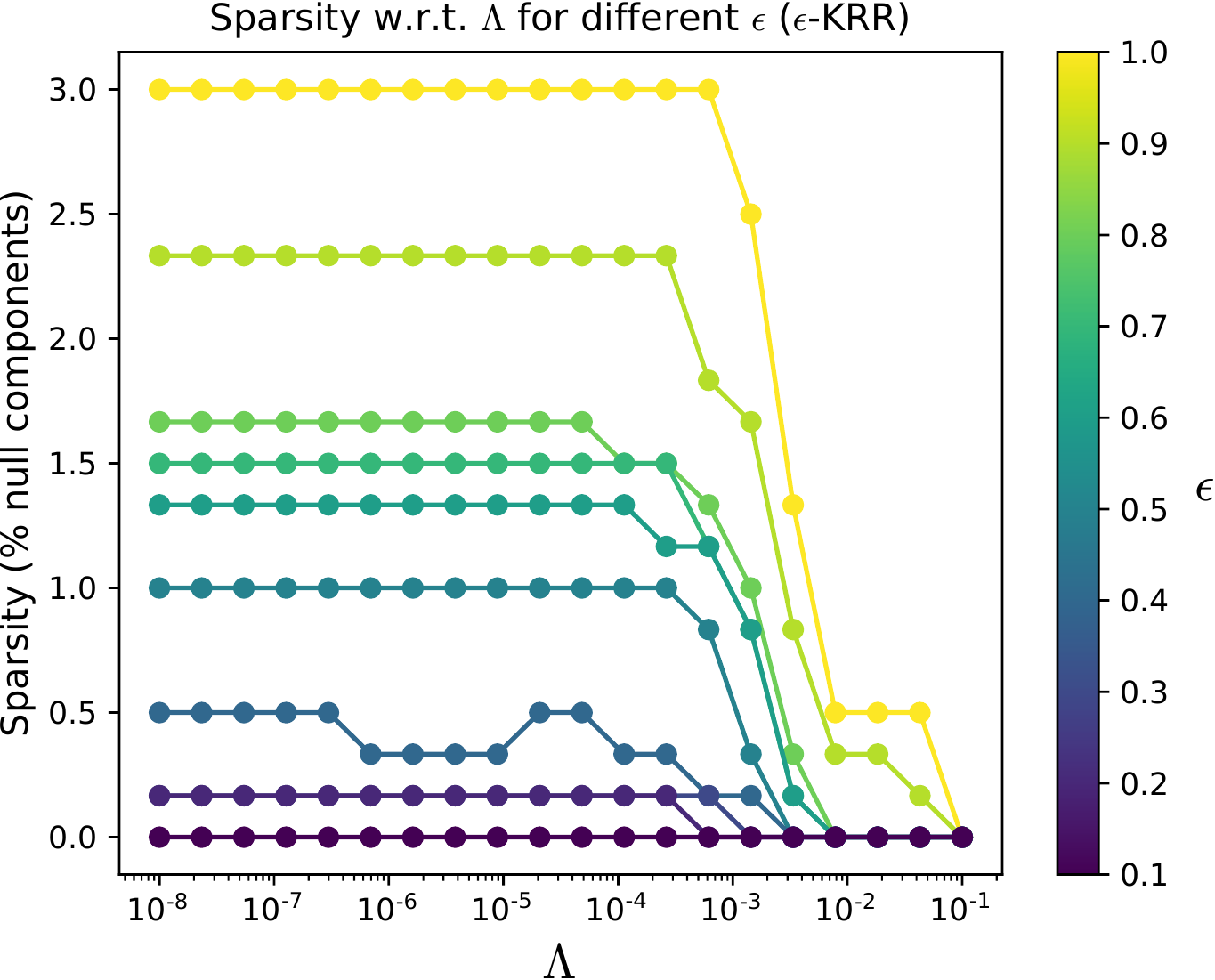}
\vspace{-0.3cm}
\caption{MSE and Sparsity w.r.t. $\Lambda$ for different $\epsilon$ for the $\epsilon$-KRR on the YEAST dataset.}
\label{fig:yeast_ekrr}
\end{figure*}

\vspace{0.5cm}

\begin{figure*}[!h]
\center
\includegraphics[height=6.2cm]{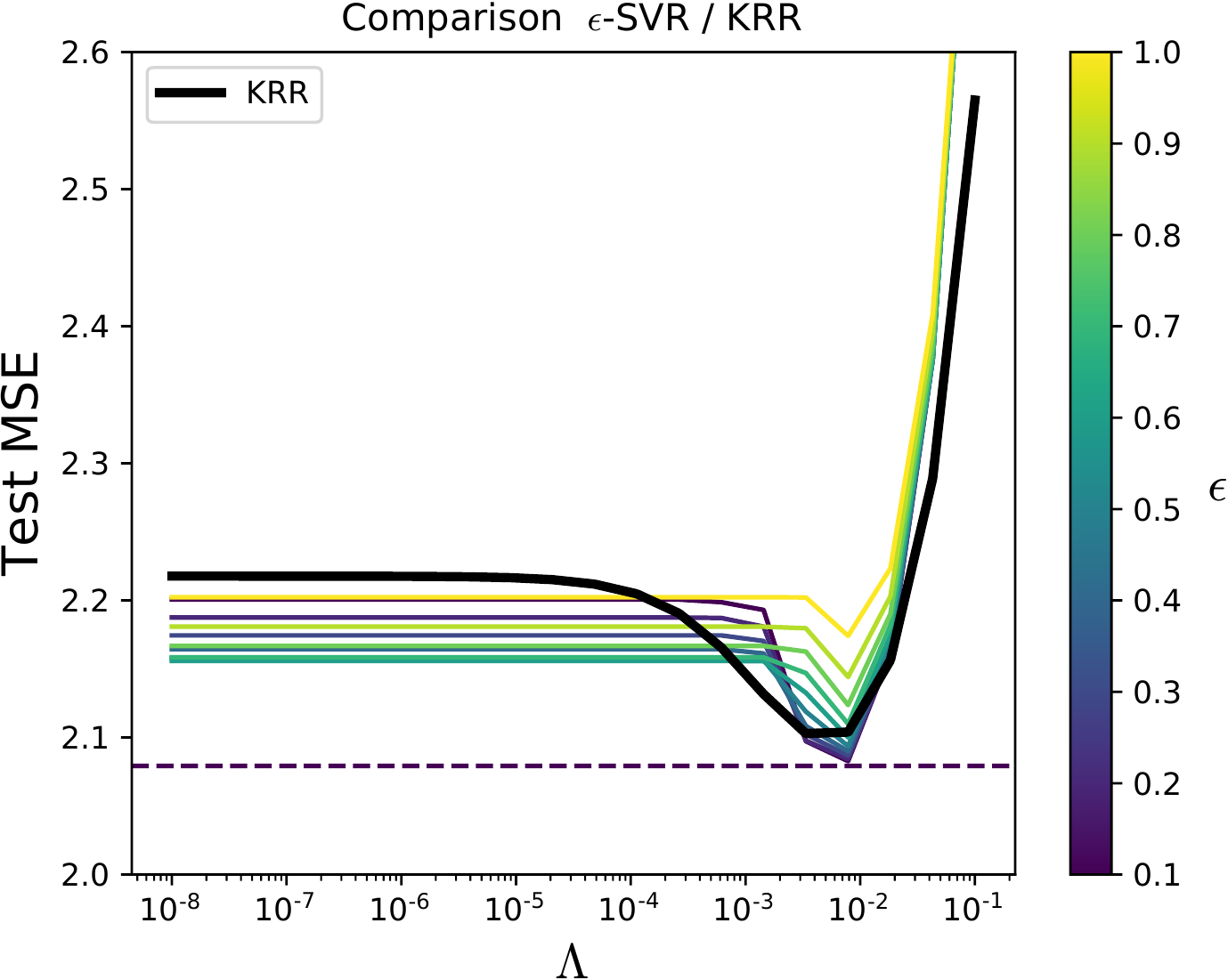}\hspace{0.5cm}
\includegraphics[height=6.2cm]{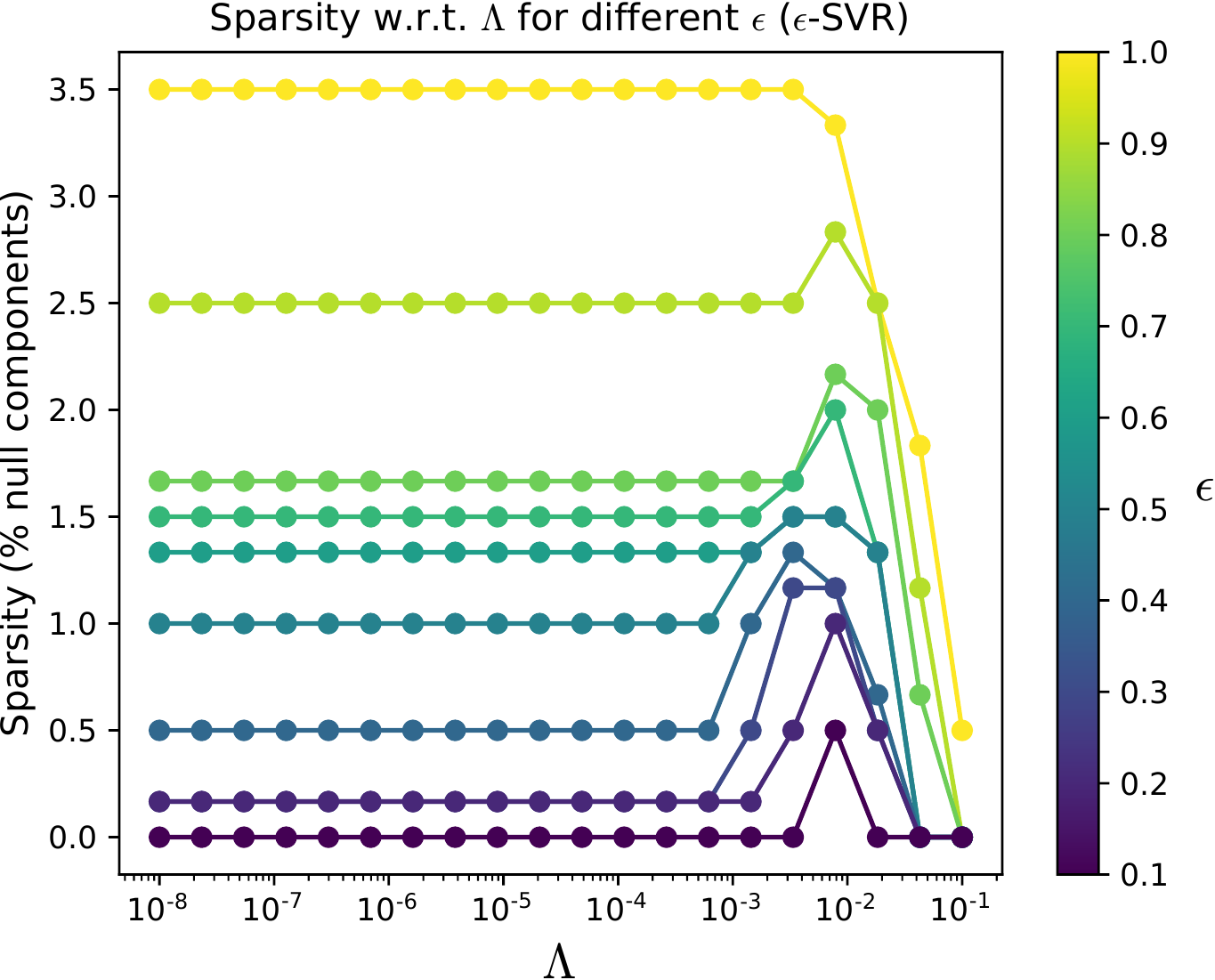}
\vspace{-0.3cm}
\caption{MSE and Sparsity w.r.t. $\Lambda$ for different $\epsilon$ for the $\epsilon$-SVR on the YEAST dataset.}
\label{fig:yeast_esvr}
\end{figure*}

\vspace{0.5cm}

\begin{figure*}[!h]
\center
\includegraphics[height=6.2cm]{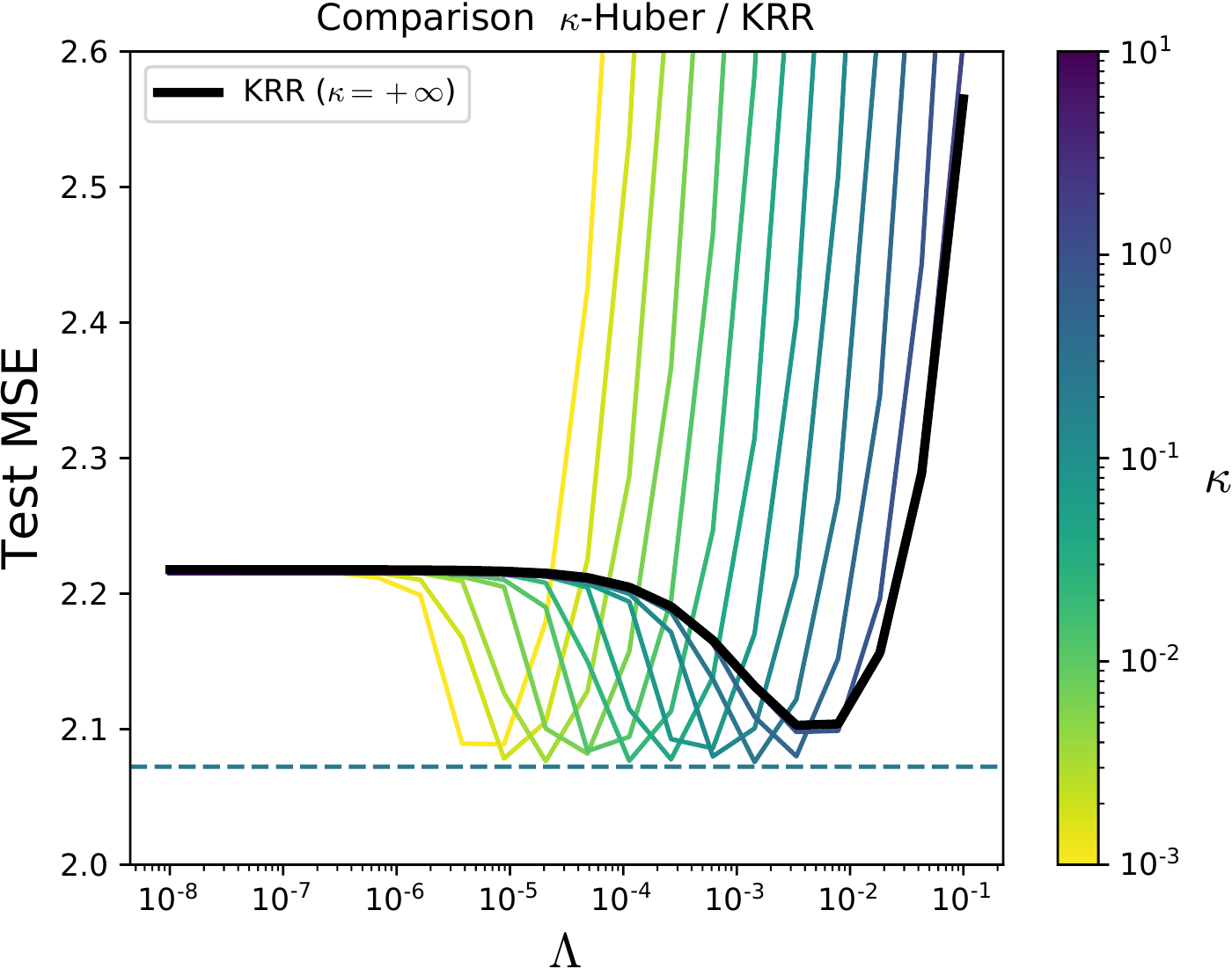}\hspace{0.5cm}
\includegraphics[height=6.2cm]{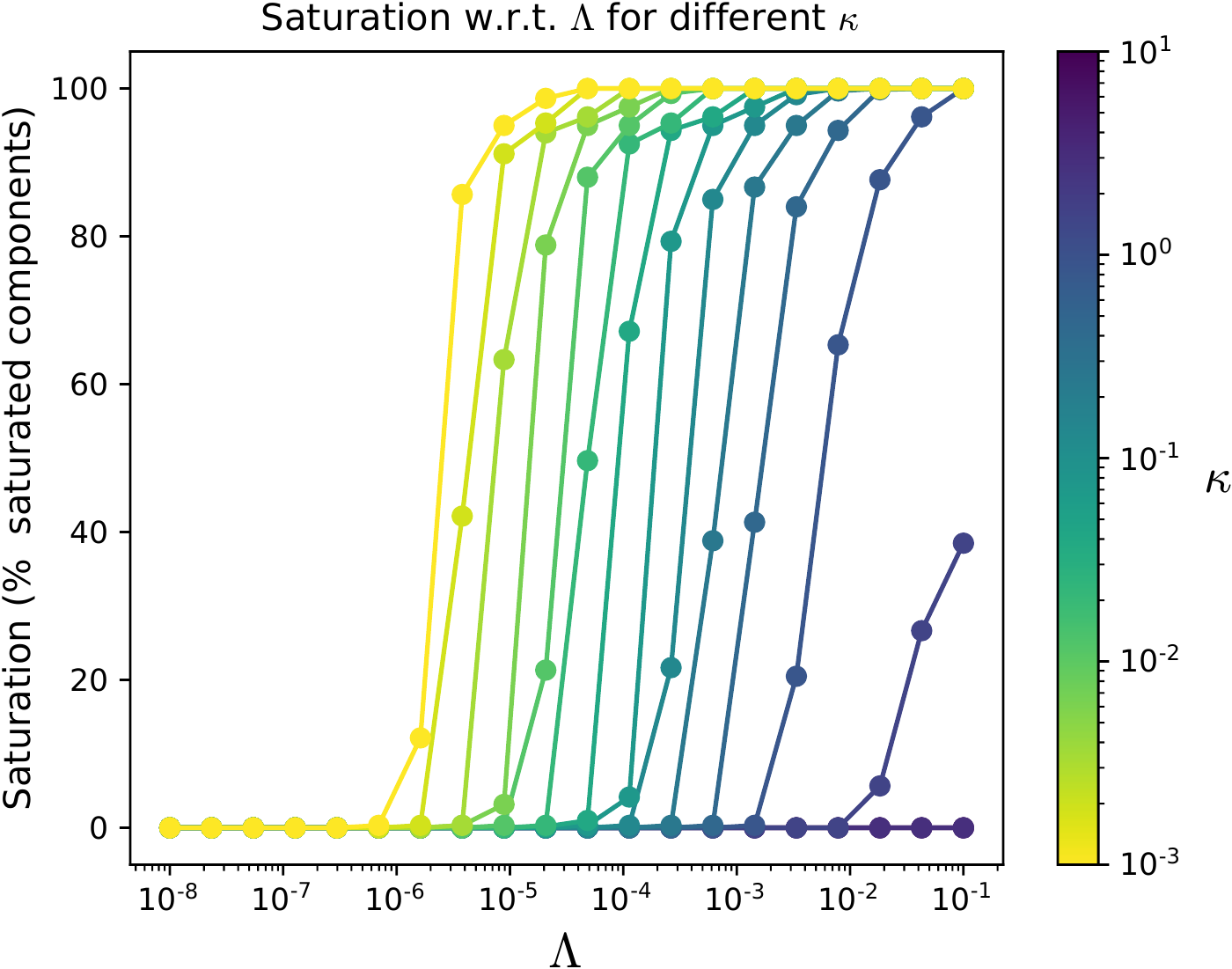}
\vspace{-0.3cm}
\caption{MSE and Saturation w.r.t. $\Lambda$ for different $\kappa$ for the $\kappa$-Huber on the YEAST dataset.}
\label{fig:yeast_huber}
\end{figure*}

\bibliographystyleNew{apalike}
\bibliographyNew{ref2}


\end{document}